%% file: main.tex
\documentclass{article}

\usepackage[preprint]{neurips_2026}

\usepackage[utf8]{inputenc}
\usepackage[T1]{fontenc}
\usepackage{hyperref}
\usepackage{url}
\usepackage{booktabs}
\usepackage{amsfonts}
\usepackage{amsmath}
\usepackage{amssymb}
\usepackage{nicefrac}
\usepackage{microtype}
\usepackage{xcolor}
\usepackage{tcolorbox}
\tcbuselibrary{breakable}
\usepackage{graphicx}
\usepackage{multirow}
\usepackage{subcaption}
\usepackage{algorithm}
\usepackage{algpseudocode}
\usepackage{amsthm}
\usepackage{enumitem}

\newtheorem{theorem}{Theorem}
\newtheorem{definition}{Definition}
\newtheorem{proposition}{Proposition}
\newtheorem{corollary}{Corollary}

\newtheorem{remark}{Remark}

\DeclareMathOperator*{\E}{\mathbb{E}}

\title{TEMPO: Temporal Enforcement via Mode-Separated Policy Optimization for Trustworthy LLM Backtesting}

\author{%
  Zeyu Zhang\thanks{Corresponding author} \\
  Department of Statistic and Data Science\\
  Northwestern University\\
  \texttt{zeyuzhang2028@u.northwestern.edu} \\
  \And
  Bradly C. Stadie \\
  Department of Statistic and Data Science \\
  Northwestern University \\
  \texttt{bstadie@northwestern.edu} \\
}

\begin{document}

\maketitle

\begin{abstract}
  \input{sections/Abstract}
\end{abstract}

\input{sections/Introduction}
\input{sections/RelatedWork}
\input{sections/Methodology}
\input{sections/ExperimentSetup}
\input{sections/ResultsAndAnalysis}
\input{sections/Conclusion}

\newpage

\bibliographystyle{plainnat}
\bibliography{references}

\newpage
\appendix

\input{sections/Appendix}


\end{document}

%% file: sections/Abstract.tex
Backtesting large language models on historical events requires reasoning exclusively from information available before a specified cutoff date.
Yet models routinely leak post-cutoff knowledge from pre-training into their reasoning, inflating apparent accuracy and undermining evaluation validity.
Prompt-based constraints fail when suppressed content is causally related to the prediction, and knowledge unlearning cannot address this problem because temporal compliance is instance-specific: the same fact may be legitimate evidence for one cutoff date and a violation for another.
Rather than erasing knowledge, the model must learn \emph{temporal discipline}---selecting evidence conditioned on each instance's cutoff date.
We propose TEMPO (Temporal Enforcement via Mode-separated Policy Optimization), which trains this discipline via two contributions: (1)~a two-mode reward where a leakage mode drives post-cutoff claims to zero as a hard prerequisite before a performance mode optimizes task performance; and (2)~a GRPO-based training pipeline that enables the model to discover temporally valid reasoning strategies.
We prove that training monotonically decreases leakage, converges to the leak-free optimum, and improves task performance once compliance is achieved.
On three prediction tasks and two models, TEMPO reduces leakage from 2--13\% to 0.6--3.7\% across all conditions, with task performance improving 6--13\% where strong pre-cutoff signals exist and maintained where the prediction task is inherently difficult from valid information alone.

%% file: sections/Introduction.tex
\section{Introduction}
\label{sec:intro}

Large language models are increasingly used in prediction tasks across diverse domains, including finance~\citep{wu2023bloomberggpt,fu2025newquant}, law~\citep{shui2023comprehensive,deng2024adapt}, and sports~\citep{mendesneves2024lem}.
Evaluating prediction ability at scale requires \emph{backtesting}: testing on historical events whose outcomes are already known so that accuracy can be measured immediately~\citep{halawi2024forecasting}.
Yet because these events have already occurred, their outcomes may be encoded in the model's pre-training data, allowing the model to recall actual outcomes rather than reason from evidence available at the time---inflating apparent accuracy~\citep{paleka2026pitfalls}.
This \emph{temporal knowledge leakage} differs from traditional data contamination~\citep{carlini2021extracting}: the model need not reproduce verbatim text but exploits general post-cutoff knowledge encoded in parametric memory~\citep{wallat2024temporal,li2025profit}.
Prompt-based constraints and standard post-training fail to suppress it~\citep{gao2025prompts,kocyigit2026impact}.
Backtesting validity therefore depends on strict temporal discipline: the model must reason as an observer at time~$t$, with no access to later information.

Existing mitigations---knowledge unlearning~\citep{blanco2024unlearning,gandikota2024elm}, chronologically consistent pretraining~\citep{he2025chronobert}, temporal abstention~\citep{zhou2026silence}, and inference-time verification~\citep{zhang2026leakscountcountmore}---each address partial aspects but none trains the model to determine which evidence is temporally valid for a given cutoff date.
Temporal compliance is instance-specific: the same fact may be legitimate for one cutoff and a violation for another, so the target set cannot be enumerated.
No existing method trains the model to behaviorally avoid post-cutoff evidence while maintaining prediction quality.

We observe that temporal leakage is not a knowledge-storage problem but a knowledge-use problem: the model does not need to forget post-cutoff facts but must learn \emph{temporal discipline}---conditioning its evidence selection on each instance's cutoff date so that only pre-cutoff information enters the reasoning.
We propose \textbf{TEMPO} (\textbf{T}emporal \textbf{E}nforcement via \textbf{M}ode-separated \textbf{P}olicy \textbf{O}ptimization), which trains temporal discipline into the model via reinforcement learning.
TEMPO introduces a \textit{two-mode reward design}: a group-level mode gate partitions each training step into a \emph{leakage mode} that drives post-cutoff claims to zero as a hard prerequisite, and a \emph{performance mode} that optimizes task performance subject to evidence-coverage constraints once compliance is achieved.
Because any leakage invalidates the evaluation, leakage elimination must be a hard gate rather than a soft penalty (Appendix~\ref{sec:app_design_justification}).
TEMPO fine-tunes via Group Relative Policy Optimization (GRPO;~\citealt{shao2024deepseekmath}) with LoRA adapters, enabling the model to discover temporally valid reasoning strategies, substituting pre-cutoff fundamentals for memorized outcomes.
We prove that training under this reward monotonically decreases leakage, converges linearly to the leak-free optimum under a regularity condition, and improves task performance once compliance is achieved (Theorem~\ref{thm:convergence}).

On three prediction tasks (stock ranking, salary estimation, legal outcome prediction) and two model generations, TEMPO reduces leakage from 2--13\% to 0.6--3.7\% across all conditions.
Once leak-free, the training curriculum naturally encourages the model to explore reasoning paths grounded in temporally valid evidence, improving task performance 6--13\% on tasks where strong pre-cutoff signals exist and maintaining it on tasks that are inherently difficult to predict from pre-cutoff information alone.

%% file: sections/RelatedWork.tex
\section{Related work}
\label{sec:related}

\paragraph{LLM Prediction and Backtesting.}
LLMs increasingly approach human-level prediction accuracy across forecasting~\citep{halawi2024forecasting,schoenegger2025wisdom,karger2025forecastbench}, finance~\citep{wu2023bloomberggpt,li2025profit}, law~\citep{shui2023comprehensive}, and sports~\citep{mendesneves2024lem}.
Backtesting---the dominant evaluation paradigm~\citep{paleka2026pitfalls}---rests on a strict temporal assumption: the model must reason exclusively from pre-cutoff information.
When this assumption is violated, back-tested performance conflates genuine prediction ability with temporal leakage.

\paragraph{Data Contamination and Temporal Leakage.}
Standard contamination mitigations---dynamic benchmarks~\citep{karger2025forecastbench}, deduplication, membership-inference detection~\citep{carlini2021extracting,xu2024benchmarking}---target instance-level overlap.
Temporal leakage is a harder sub-problem: the model encodes \emph{general post-cutoff knowledge} through parametric memory without verbatim reproduction~\citep{wallat2024temporal,zhu2025evolvebench}, so standard defenses and prompt-based constraints do not apply~\citep{gao2025prompts,kocyigit2026impact}.
Knowledge unlearning~\citep{blanco2024unlearning,gandikota2024elm,yin2024temporal,reisizadeh2025blur} cannot enumerate the instance-specific target set and correlated knowledge reconstructs erased content~\citep{wang2025uipe}; chronologically consistent pretraining~\citep{he2025chronobert} and temporal abstention~\citep{zhou2026silence} require $O(T)$ checkpoints or teach refusal rather than prediction; inference-time verification~\citep{zhang2026leakscountcountmore} filters claims post-hoc but leaves the model's behavior unchanged.
TEMPO is the first method that trains per-instance temporal discipline into the model's behavior.

\paragraph{GRPO and RL for Behavioral Compliance.}
GRPO~\citep{shao2024deepseekmath} has become the standard RL algorithm for LLM post-training, validated at scale by DeepSeek-R1~\citep{deepseekr1} and stabilized by DAPO~\citep{yu2025dapo}.
When multiple objectives compete, one reward can dominate; MO-GRPO~\citep{ichihara2025mogrpo}, GDPO~\citep{liu2026gdpo}, and constrained GRPO~\citep{girgis2026cgrpo} mitigate this but all permit soft trade-offs.
RL for behavioral compliance---Safe RLHF~\citep{dai2024saferlhf}, TruthRL~\citep{wei2025truthrl}, CaRR~\citep{zhang2026carr}---learns fixed rules (a claim is always safe or unsafe regardless of context) with soft trade-offs.
Two gaps remain: (1)~no method enforces a \emph{hard constraint} where any violation invalidates the evaluation, and (2)~no method handles \emph{instance-specific rules} where compliance depends on a per-instance variable.
TEMPO fills both gaps with a binary mode gate, and we prove convergence to the leak-free optimum while interleaving instance-specific mode assignments within each batch.

%% file: sections/Methodology.tex
\section{Methodology}
\label{sec:method}

Temporal discipline---conditioning evidence selection on a per-instance cutoff date---is a behavioral constraint that must be learned through training.
This section formalizes the problem (Section~\ref{sec:method_problem}), introduces the two-mode reward that separates leakage elimination from performance optimization (Section~\ref{sec:reward}), and presents the GRPO-based training algorithm with convergence guarantees (Section~\ref{sec:grpo}).
Figure~\ref{fig:pipeline} illustrates the full pipeline.

\begin{figure}[t]
  \centering
  \includegraphics[width=0.95\textwidth]{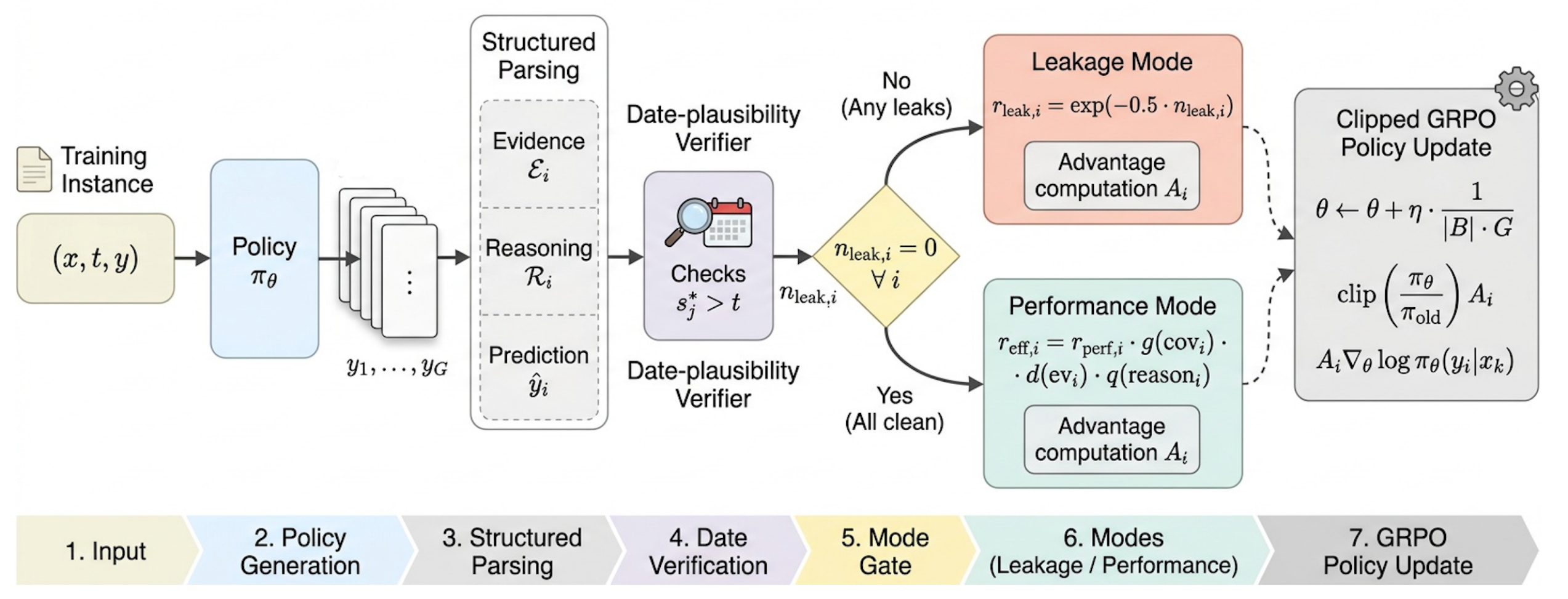}
  \caption{Overview of one TEMPO training step. For each training instance $(x, t, y)$, the policy $\pi_\theta$ generates $G$ completions, each parsed into a structured output with an evidence list, reasoning, and prediction. A date-plausibility verifier counts leaked claims per completion. The core design contribution is the \emph{group-level mode gate}: if any completion in the group contains a leaked claim, the entire group enters \emph{leakage mode} with an exponential penalty reward (Eq.~\ref{eq:rleak}); only when all completions are clean does the group enter \emph{performance mode} to optimize task accuracy (Eq.~\ref{eq:reff}). This hard separation ensures leakage elimination is a prerequisite for performance optimization. The resulting rewards drive a clipped GRPO policy update.}
  \label{fig:pipeline}
\end{figure}

\subsection{Problem formulation}
\label{sec:method_problem}

Let $\pi_\theta$ denote the policy (an LLM parameterized by $\theta$).
A training instance is a tuple $(x, t, y)$: a prompt $x$ containing the task description and case-specific input, a \emph{cutoff date} $t$ specifying the last calendar date whose information may legitimately inform the answer, and a ground-truth label $y$.
We operate in the backtesting regime where $y$ is already resolved, so both the task-performance reward and the temporal-leakage penalty are computable at every training step---a property that makes reinforcement learning directly applicable.

The policy produces a structured completion: an evidence list $\mathcal{E} = \{(c_j, s_j)\}_{j=1}^{N}$ of atomic claims with declared source dates, a reasoning segment $\mathcal{R}$ citing evidence by index, and a prediction $\hat{y}$.
This structure enables per-claim temporal verification: each evidence item carries a date checkable against the cutoff.

\begin{definition}[Temporal leakage]
\label{def:leakage}
Claim $c_j$ \emph{leaks} relative to cutoff $t$ if its effective source date $s_j^{*} > t$, where $s_j^{*}$ is the declared date $s_j$ adjusted by a date-plausibility verifier (Appendix~\ref{sec:app_implementation}).
The \emph{leaked claim count} for completion~$y$ at cutoff~$t$ is
\[
  c(y,t) \;=\; \bigl\lvert\{j : s_j^{*}(y) > t\}\bigr\rvert.
\]
We write $n_{\mathrm{leak},i} = c(y_i, t)$ for the $i$-th completion in a group.
A completion is \emph{clean} if $n_{\mathrm{leak},i} = 0$.
\end{definition}

Following GRPO~\citep{shao2024deepseekmath}, we sample a group of $G$ completions $\{y_1, \ldots, y_G\} \sim \pi_\theta(\cdot \mid x)$ for each training instance.
Let $q(\theta; x, t) = \Pr_{y \sim \pi_\theta}[c(y,t) > 0]$ denote the per-instance leakage probability---the chance that a single sampled completion contains at least one leaked claim.
The probability that \emph{all} $G$ completions are clean is $p_0 = (1{-}q)^G$, and the group leakage indicator is $L = \mathbf{1}[\exists\,i : n_{\mathrm{leak},i} > 0]$.

TEMPO maximizes the mode-gated expected reward:
\begin{equation}
  \label{eq:objective}
  J_{\mathrm{TEMPO}}(\theta)
  \;=\; \E_{(x,t,y)}\!\bigl[
    p_0\;\E[r_{\mathrm{eff}} \mid L{=}0]
    \;+\; (1{-}p_0)\;\E[r_{\mathrm{leak}} \mid L{=}1]
  \bigr],
\end{equation}
where $r_{\mathrm{leak}}$ and $r_{\mathrm{eff}}$ are mode-specific rewards (Section~\ref{sec:reward}).
Since $p_0 = (1{-}q)^G$ decays exponentially in~$G$, the objective implicitly prioritizes leakage elimination: any positive $q$ suppresses performance-mode contribution.

\subsection{Two-mode reward design}
\label{sec:reward}

A naive scalar combination $r_{\mathrm{mix}} = \alpha\, r_{\mathrm{leak}} + (1{-}\alpha)\, r_{\mathrm{perf}}$ fails because performance variance dominates the z-scored advantage and the leakage signal vanishes; moreover, any leakage invalidates the evaluation, so leakage elimination must be a hard prerequisite rather than a soft trade-off (Appendix~\ref{sec:app_design_justification}).
We therefore separate the two objectives via a group-level mode gate.

\textbf{Mode Selection Gate.}
The group enters \emph{performance mode} if and only if every completion is clean:
\begin{equation}
  \label{eq:mode}
  \text{mode} =
  \begin{cases}
    \text{\textsc{performance}} & \text{if } n_{\mathrm{leak},i} = 0\;\forall\, i \in \{1,\ldots,G\}, \\
    \text{\textsc{leakage}}     & \text{otherwise.}
  \end{cases}
\end{equation}
The gate operates at group level because GRPO computes advantages as within-group contrasts: all completions in a group must be evaluated under the same objective for the gradient to be coherent.
Early in training, most groups contain at least one leaked completion and enter leakage mode; as training drives $q$ toward zero, groups progressively transition to performance mode---an emergent curriculum characterized in Section~\ref{sec:grpo}.

\textbf{Leakage Mode.}
When the group enters leakage mode, the reward drives down leaked evidence before competing on task performance.
Each completion receives a reward that is strictly decreasing in the number of leaked claims:
\begin{equation}
  \label{eq:rleak}
  r_{\mathrm{leak},i} = \exp\bigl(-0.5 \cdot n_{\mathrm{leak},i}\bigr).
\end{equation}
The exponential form provides graded signal: completions with fewer leaked claims receive higher reward, and the gap between clean ($r_{\mathrm{leak}} = 1.0$) and single-leak ($r_{\mathrm{leak}} = 0.61$) completions remains large enough for clear gradient signal even when most completions are already clean.

\textbf{Performance Mode.}
Once every completion is clean, the group optimizes task performance.
The effective reward is
\begin{equation}
  \label{eq:reff}
  r_{\mathrm{eff},i} = r_{\mathrm{perf},i} \cdot g(\mathrm{cov}_i) \cdot d(\mathrm{ev}_i) \cdot w(\mathrm{reason}_i),
\end{equation}
where $r_{\mathrm{perf},i} \in [0,1]$ is a task-specific performance score (Section~\ref{sec:eval_protocol}), and $g$, $d$, $w$ are multiplicative quality gates for evidence citation, diversity, and reasoning depth (Appendix~\ref{sec:app_reward_details}).
These gates prevent degenerate compliance strategies such as producing minimal evidence to avoid leakage.

\subsection{TEMPO optimization}
\label{sec:grpo}

We optimize $\pi_\theta$ with GRPO~\citep{shao2024deepseekmath}, which z-scores rewards within each group of $G$ completions to compute advantages without a critic network.
A batch-level baseline calibrates advantages across prompts, and a per-token KL penalty ($\beta_{\mathrm{KL}}$) against the frozen base model prevents excessive drift (masked to content tokens only).
The full training loop appears in Algorithm~\ref{alg:tempo} (Appendix~\ref{sec:app_implementation}).

\textbf{Convergence Guarantees.}
Standard GRPO convergence theory~\citep{shao2024deepseekmath,ghadimi2013stochastic} applies to single-objective optimization.
TEMPO's two-mode design introduces two structural complications that require new analysis:
(1)~the leakage-mode reward $r_{\mathrm{leak}} = \exp(-0.5\,c)$ is a nonlinear transformation of the leaked claim count---does maximizing it actually reduce leakage?
(2)~leakage-mode and performance-mode updates interleave within each batch---can performance-mode updates undo leakage progress?
The following two propositions address these gaps; once both are established, convergence follows from the standard stochastic descent framework.

Let $V_c(\theta) = \E_{(x,t)}\E_{y \sim \pi_\theta}[c(y,t)]$ be the expected leakage cost and $V_r(\theta) = \E_{(x,t)}\E_{y \sim \pi_\theta}[r_{\mathrm{eff}}(y) \mid c(y,t){=}0]$ the expected performance reward conditional on clean outputs.
We assume $V_c$ is $L_c$-smooth, $V_r$ is $L_r$-smooth, and $\E[\|\hat{g}_k\|^2] \leq B_g^2$.

\begin{proposition}[Gradient alignment]
\label{prop:alignment}
Let $V_f(\theta) = \E_{(x,t)}\E_{y \sim \pi_\theta}[\exp(-0.5\,c(y,t))]$ denote the expected exponential reward.
Let $\hat{g}_{\mathrm{leak}}$ be the stochastic GRPO gradient computed from leakage-mode groups, and let $\bar{g}_{\mathrm{leak}}(\theta) = \E[\hat{g}_{\mathrm{leak}} \mid \theta]$ denote its expectation.
Under the compatible approximation condition~\citep{kakade2001natural}, the following holds whenever $\mathrm{Var}_{\pi_\theta}(c) > 0$:
\begin{equation}
  \label{eq:alignment}
  \langle \bar{g}_{\mathrm{leak}},\; \nabla V_c \rangle \;\leq\; -\gamma(\theta)\,\|\nabla V_c\|^2,
\end{equation}
where $\gamma(\theta) > 0$ is a policy-dependent alignment constant.
That is, the GRPO gradient under the exponential leakage reward is a provable descent direction for the expected leakage cost (proof in Appendix~\ref{sec:app_proof_alignment}).
\end{proposition}

\begin{proposition}[Bounded cross-mode contamination]
\label{prop:crossmode}
Let $\alpha_k \in [0,1]$ be the fraction of groups in leakage mode at step~$k$, and let $\hat{g}_{\mathrm{perf}}$ be the stochastic GRPO gradient from performance-mode groups.
The mixed-batch update $\theta_{k+1} = \theta_k + \eta\,(\alpha_k\hat{g}_{\mathrm{leak}} + (1{-}\alpha_k)\hat{g}_{\mathrm{perf}})$ satisfies
\begin{equation}
  \label{eq:crossmode}
  \E[V_c(\theta_{k+1}) \mid \theta_k] \;\leq\; V_c(\theta_k) - \underbrace{\bigl(\alpha_k\gamma(\theta_k) - (1{-}\alpha_k)B_g/\|\nabla V_c(\theta_k)\|\bigr)}_{\text{effective descent rate}}\,\eta\|\nabla V_c(\theta_k)\|^2 + \tfrac{L_c\eta^2}{2}B_g^2.
\end{equation}
The effective descent rate is positive whenever the leakage gradient (weighted by $\alpha_k$) dominates the worst-case cross-mode interference from performance-mode updates (proof in Appendix~\ref{sec:app_proof_crossmode}).
\end{proposition}

With both gaps filled, convergence follows from the standard stochastic optimization framework:

\begin{theorem}[TEMPO convergence]
\label{thm:convergence}
Under Propositions~\ref{prop:alignment}--\ref{prop:crossmode}, let $\gamma_0 = \inf\{\gamma(\theta) : V_c(\theta) \geq \epsilon_0\} > 0$ for a target threshold $\epsilon_0$, and suppose at least a fraction $\alpha_{\min} > 0$ of groups are in leakage mode whenever $V_c(\theta_k) \geq \epsilon_0$.
For step size $\eta \leq \gamma_0/(L_c B_g^2)$:
(i)~leakage decreases at the standard nonconvex rate $\min_{k \leq T}\|\nabla V_c\|^2 = O(1/\sqrt{T})$;
(ii)~performance improves during performance-mode updates by the standard ascent lemma;
(iii)~under the Polyak--\L{}ojasiewicz condition $\|\nabla V_c\|^2 \geq 2\mu\, V_c$, leakage converges linearly:
\begin{equation}
  \label{eq:linear-conv}
  \E[V_c(\theta_k)] \;\leq\; (1 - \alpha_{\min}\eta\gamma_0\mu)^k\, V_c(\theta_0) \;+\; \frac{L_c\eta B_g^2}{2\alpha_{\min}\gamma_0\mu}.
\end{equation}
\end{theorem}
\noindent Complete derivations appear in Appendix~\ref{sec:app_convergence_proofs}.
A practical consequence of Theorem~\ref{thm:convergence} is that training does not require manual scheduling of when to switch from leakage elimination to performance optimization---the mode gate handles this automatically:

\begin{corollary}[Two-phase curriculum]
\label{cor:curriculum}
Recall that the group-clean probability is $\bar{p}_0(\theta) = \E_{(x,t)}[(1{-}q)^G]$.
Since $\bar{p}_0(\theta) \geq (1 - V_c(\theta))^G$ by Markov's inequality, the linear convergence of $V_c$ in Theorem~\ref{thm:convergence}(iii) implies $\bar{p}_0 \to 1$: as leakage decreases, progressively more groups become fully clean and enter performance mode, so training automatically transitions from a leakage-dominated phase to a performance-optimization phase.
\end{corollary}

%% file: sections/ExperimentSetup.tex
\section{Experimental setup}
\label{sec:experiments}

We evaluate TEMPO on three prediction tasks spanning financial markets, professional sports contracts, and federal law, comparing against inference-time baselines under an identical evaluation pipeline.

\subsection{Tasks and data}
\label{sec:data}

The three tasks are chosen to cover a spectrum of leakage sensitivities.
Stock ranking is the most leakage-sensitive: market returns are highly volatile and difficult to predict from pre-cutoff fundamentals alone, so models with access to post-cutoff knowledge gain a disproportionate advantage.
Salary prediction is intermediate: contract values depend on recent player performance and market conditions that may fall within days or weeks of the cutoff, making up-to-date knowledge informative but not as decisive as in stock markets.
Legal case prediction is least susceptible: judicial outcomes rely primarily on precedent, statutory reasoning, and procedural history available well before the decision, and similar prior cases provide strong valid reference points.
Each instance carries a cutoff date equal to the date its oracle result became publicly determinable; only information available on or before this date is considered temporally valid.
Training and evaluation instances never overlap: stock evaluation uses a temporally disjoint window, salary evaluation uses a held-out sport (NBA, absent from training), and legal evaluation is drawn by stratified sampling with verified zero case overlap.

\textbf{Stock Ranking.}
Rank five same-sector companies by six-month total return; evaluation covers COVID-19 onset (Dec 2019--Jun 2020).

\textbf{Salary Prediction.}
Predict the average annual contract value (AAV) for an athlete who changed teams; training spans NFL/MLB/NHL, evaluation uses NBA.

\textbf{Legal Case Prediction.}
Predict the probability the petitioner prevails given case facts and procedural history; training combines Supreme Court and federal appellate cases (2014--2024).

A separate control set per task (answers postdating Feb 2026) establishes a zero-knowledge performance floor.
Table~\ref{tab:dataset_summary} summarizes statistics; dataset details are in Appendix~\ref{sec:app_dataset}.

\input{tables/dataset_summary}

\subsection{Models and methods}
\label{sec:methods_compare}

We train two mixture-of-experts models from the Qwen3 family~\citep{qwen3}: Qwen3-30B-A3B (30B total, 3B active per token, knowledge cutoff approximately March 2025) and Qwen3.5-35B-A3B (35B total, 3B active, cutoff approximately February 2026).
Both share the same architecture family but differ in scale and pre-training vintage, allowing us to test whether temporal discipline generalizes across model generations.
All methods receive identical data, prompts, and evaluation pipelines.\footnote{Code and data are available at \url{https://anonymous.4open.science/r/TEMPO-7EA2}.}

We compare against three inference-time baselines that form a progression of intervention levels, testing whether each additional mechanism is sufficient to suppress temporal leakage:

\textbf{Temporal Hint (Base).}
The unmodified model receives a structured prompt with a temporal constraint anchored to the instance cutoff date.
This tests whether instruction alone suffices, without retrieval or verification.

\textbf{RAG.}
The base prompt is augmented with pre-cutoff evidence retrieved via a search API with a hard date filter set to the instance cutoff.
This tests whether providing external valid evidence reduces reliance on parametric memory.

\textbf{TimeSPEC.}
An inference-time leakage prevention approach~\citep{zhang2026leakscountcountmore} that extends RAG with a post-generation verify--regenerate--aggregate pipeline.
This tests whether inference-time claim filtering achieves compliance.

TEMPO differs fundamentally: rather than filtering at inference time, it trains temporal discipline into the model's parameters (Section~\ref{sec:method}), requiring only a single forward pass at evaluation.
Training requires approximately 6 hours per model-task combination on four concurrent workers; evaluation uses greedy decoding (single deterministic forward pass per instance), so all reported metrics are fully reproducible given the released checkpoints.
Implementation details and hyperparameters for all methods appear in Appendix~\ref{sec:app_implementation}; the TimeSPEC pipeline is documented in Appendix~\ref{sec:app_timespec}.

\subsection{Evaluation protocol}
\label{sec:eval_protocol}

All methods are scored by an identical pipeline to ensure fair comparison.
We report three complementary metrics (formal definitions in Appendix~\ref{sec:app_eval_pipeline}).

\textbf{Overall Leakage Rate (OLR).}
$\mathrm{OLR} = N^{-1}\sum_{j} n_{\mathrm{leak},j}/n_{\mathrm{total},j}$, the fraction of claims whose effective source date exceeds the cutoff.
A date-plausibility verifier adjusts declared dates for temporal consistency, mirroring the training signal (Section~\ref{sec:reward}).
Lower is better; 0\% means every claim is grounded in pre-cutoff information.
Control experiments on post-cutoff instances (where the model lacks parametric knowledge of the outcome) confirm the verifier produces negligible false positives (0.0--0.7\% OLR; Appendix~\ref{sec:app_control}).

\textbf{Task Performance (Perf).}
Task-specific performance mapped to $[0,1]$: Spearman rank correlation $(\rho + 1)/2$ for stock, relative accuracy $\max(0, 1 - |\hat{y} - y|/y)$ for salary, and Brier score complement $1 - (\hat{p} - y)^2$ for legal, where $\hat{y}$ (or $\hat{p}$) denotes the model's prediction and $y$ the ground truth.
Performance is interpretable only alongside OLR: high performance at high OLR may reflect contamination rather than genuine prediction ability.
Control experiments on post-cutoff instances establish the zero-knowledge floor.

\textbf{Coverage (Cov).}
$\mathrm{Cov} = |\{e_k : e_k \text{ cited in reasoning}\}| / |\{e_k\}|$, the fraction of evidence items cited at least once in the reasoning; high coverage reflects a coherent evidence--reasoning--prediction chain.

%% file: tables/dataset_summary.tex
\begin{table}[t]
  \caption{Dataset statistics across three prediction tasks. All training and evaluation cutoff dates predate March 2025, ensuring both models could have encountered the outcomes during pre-training. Control instances postdate February 2026, after both models' knowledge cutoffs, establishing a zero-knowledge performance floor.}
  \label{tab:dataset_summary}
  \centering
  \small
  \begin{tabular}{l r r r l l}
    \toprule
    Task  & Train & Eval & Control & Train period & Eval period \\
    \midrule
    Stock ranking                & 1{,}200 & 100 & 100 & 2020-07 -- 2024-06 & 2019-12 -- 2020-06 \\
    Salary prediction         & 1{,}066 & 137 &  26 & 2017-12 -- 2024-12 & 2019-06 -- 2024-07 \\
    Legal case         &     894 & 100 &  90 & 2014-10 -- 2024-12 & 2015-01 -- 2024-07 \\
    \bottomrule
  \end{tabular}
\end{table}

%% file: sections/ResultsAndAnalysis.tex
\section{Results and analysis}
\label{sec:results}

Table~\ref{tab:main_results} reports the full results and Figure~\ref{fig:method_2d} visualizes the leakage--performance space.
TEMPO reduces leakage from 2--13\% (baselines) to 0.6--3.7\% across all six model-task conditions while maintaining prediction performance comparable to other leakage-controlled methods and consistently exceeding the zero-knowledge control.
Neither prompt constraints (Temporal Hint) nor retrieval with hard date filters (RAG) resolve temporal leakage; RAG often amplifies it by surfacing post-cutoff content that is closely related to the prediction target.
Extended analysis---instance-level leakage and coverage distributions, inference cost comparisons, and qualitative case studies---appears in Appendices~\ref{sec:app_dynamics}--\ref{sec:app_cases}.

\input{tables/main_results}

\subsection{Per-task comparison}
\label{sec:per_task}

\textbf{Stock Ranking.}
Temporal Hint produces 5.4\%/12.7\% OLR as the model cites post-cutoff closing prices directly.
RAG amplifies leakage further (9.8\% for Qwen3-30B): even with a hard date filter, financial news discussing post-cutoff dynamics passes through, demonstrating that search cutoffs alone do not prevent contamination.
Both TimeSPEC (1.7\%/1.5\%) and TEMPO (2.5\%/2.4\%) suppress leakage effectively, but differ sharply in coverage: TimeSPEC's post-hoc filtering breaks the evidence--reasoning chain (73.8\%/75.5\%), while TEMPO generates clean evidence and reasoning together in one pass (99.3\%/95.7\%).
For Qwen3-30B, TEMPO achieves the highest accuracy among all methods (0.625) while maintaining near-zero leakage.
For Qwen3.5-35B, Temporal Hint's higher raw accuracy (0.690) comes at 12.7\% OLR---reflecting post-cutoff citations rather than genuine prediction.
Across both models, TEMPO exceeds the Control floor (0.476/0.497) by 27--31\%, confirming genuine prediction from valid pre-cutoff information.

\textbf{Salary Prediction.}
RAG achieves the highest accuracy (0.751/0.778) but at 18.2\% OLR for Qwen3-30B, as the retriever surfaces actual contract details published after the cutoff.
Both TimeSPEC and TEMPO suppress leakage to 0.7--3.7\% and achieve comparable honest prediction quality (TimeSPEC 0.638/0.656; TEMPO 0.658/0.631), with both substantially exceeding Control (0.454/0.549).
The key difference is that TEMPO improves accuracy over Temporal Hint (0.658 vs.\ 0.607 for 30B; 0.631 vs.\ 0.558 for 35B) while suppressing leakage---its performance-mode objective actively trains better reasoning from valid pre-cutoff information---whereas TimeSPEC has no such mechanism.
The simultaneous leakage reduction and accuracy gain confirms the two-phase curriculum.

\textbf{Legal Outcome Prediction.}
All methods achieve low OLR (0.1--4.0\%) because legal reasoning draws on precedent available before decisions; the mode gate operates correctly even when baseline leakage is already near zero.
Both TEMPO and TimeSPEC substantially exceed Control (0.707/0.647), confirming genuine prediction from valid precedent.
TimeSPEC achieves higher accuracy on this task (0.824/0.811 vs.\ TEMPO's 0.804/0.741), with a larger gap for Qwen3.5-35B.
TEMPO's advantage is coverage (75.7\%/74.1\% vs.\ 65.2\%/73.0\%) and inference cost: it preserves reasoning coherence in a single forward pass without external API calls.

\begin{figure}[t]
  \centering
  \includegraphics[width=0.92\textwidth]{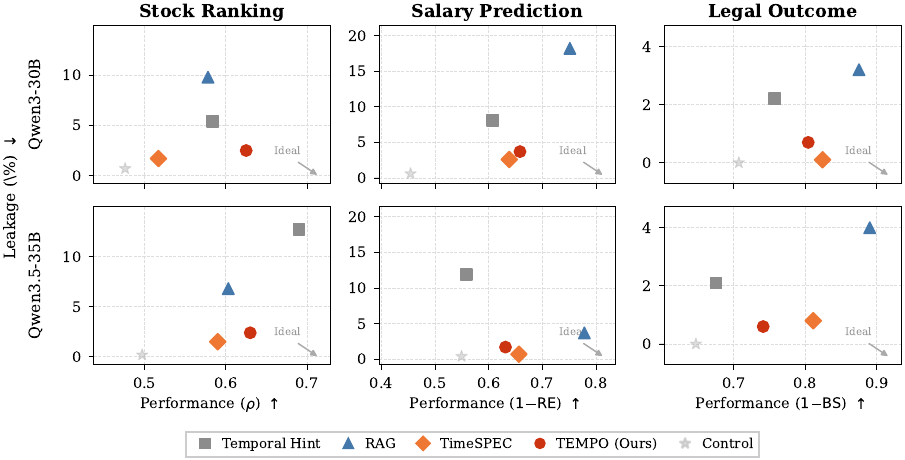}
  \caption{Leakage--performance trade-off across tasks (columns) and models (rows). The ideal region is the lower right (low leakage, high performance). TEMPO (red) consistently occupies this region. Stars denote control experiments on post-cutoff instances.}
  \label{fig:method_2d}
\end{figure}

\subsection{Temporal discipline and prediction quality}
\label{sec:2d_analysis}

Figure~\ref{fig:method_2d} plots each method in the leakage--performance space.
Baselines occupy the upper region where leakage is uncontrolled and apparent performance is unreliable.
TimeSPEC moves down---leakage suppressed---but without a performance-improvement objective it tends left, and coverage loss from filtering further degrades reasoning quality.
TEMPO moves down \emph{and} stays right or improves: its performance-mode objective trains the model to reason better from valid information while preserving 91--99\% coverage.
Across all six conditions, TEMPO suppresses leakage to 0.6--3.7\% while consistently exceeding Control by 14--45\%, confirming genuine prediction from valid pre-cutoff evidence.
Where baselines appear to outperform TEMPO in raw accuracy (RAG on salary/legal; Temporal Hint on stock 35B), the gap invariably accompanies significantly higher OLR---this ``performance'' reflects contaminated reasoning and is not trustworthy for backtesting.
On coverage, TEMPO outperforms TimeSPEC by 17--26 percentage points on stock and salary; the gap narrows on legal where fewer claims are filtered.
This pattern is what Theorem~\ref{thm:convergence} guarantees: part~(i) drives leakage down, part~(ii) improves performance conditional on compliance.

\subsection{Two-phase curriculum validation}
\label{sec:curriculum_validation}

Table~\ref{tab:mode_transition} validates the emergent two-phase curriculum predicted by Corollary~\ref{cor:curriculum} by tracking the fraction of GRPO groups assigned to performance mode over training.
On the highest-leakage condition (Qwen3.5-35B salary), training begins with 0\% of groups in performance mode---every group contains at least one leaked completion---and crosses 50\% at step~25 of 211, reaching 81.2\% by training end.
Qwen3-30B converges faster, reaching 100\% performance mode on stock and 98.8\% on salary, because lower baseline leakage means fewer groups start in leakage mode.
No manual phase scheduling is required.
Legal prediction has near-zero baseline leakage (2.1--2.2\% OLR), so the transition effect is less pronounced; the design is validated on tasks where substantial leakage exists to be eliminated.

\begin{table}[t]
  \caption{Two-mode transition over training. Performance-mode fraction: proportion of GRPO groups where all completions are leak-free. The mode gate transitions automatically, confirming Corollary~\ref{cor:curriculum}.}
  \label{tab:mode_transition}
  \centering\small
  \begin{tabular}{llcccc}
    \toprule
    Model & Task & Initial (\%) & Final (\%) & Step ${>}50$\% & Total steps \\
    \midrule
    Qwen3-30B & Stock & 62.5 & 100.0 & 0 & 213 \\
    Qwen3-30B & Salary & 50.0 & 98.8 & 2 & 211 \\
    \midrule
    Qwen3.5-35B & Stock & 12.5 & 96.2 & 11 & 213 \\
    Qwen3.5-35B & Salary & 0.0 & 81.2 & 25 & 211 \\
    \bottomrule
  \end{tabular}
\end{table}

%% file: tables/main_results.tex
\begin{table}[t]
  \caption{Evaluation results across three temporal prediction tasks and two models. OLR: overall leakage rate (fraction of claims citing post-cutoff information). Perf: task-specific prediction accuracy (mapped Spearman~$\rho$ for stock, $1{-}\mathrm{RE}$ for salary, $1{-}\mathrm{BS}$ for legal; all in $[0,1]$). Cov: fraction of evidence items cited in reasoning. Temporal Hint, RAG, and TimeSPEC are inference-time methods without parameter updates; TEMPO fine-tunes the model via GRPO. Best value per model block in \textbf{bold}. $^\dagger$Control evaluates Temporal Hint on post-cutoff instances where the model lacks parametric knowledge of the answer.}
  \label{tab:main_results}
  \centering
  \small
  \setlength{\tabcolsep}{3.5pt}
  \begin{tabular}{l ccc ccc ccc}
    \toprule
    & \multicolumn{3}{c}{Stock Ranking} & \multicolumn{3}{c}{Salary Prediction} & \multicolumn{3}{c}{Legal Outcome} \\
    \cmidrule(lr){2-4} \cmidrule(lr){5-7} \cmidrule(lr){8-10}
    Method & OLR$\downarrow$ & Perf$\uparrow$ & Cov$\uparrow$ & OLR$\downarrow$ & Perf$\uparrow$ & Cov$\uparrow$ & OLR$\downarrow$ & Perf$\uparrow$ & Cov$\uparrow$ \\
    \midrule
    \multicolumn{10}{l}{\textit{Qwen3-30B-A3B}} \\
    \midrule
    Temporal Hint   & 5.4\%              & 0.584              & 94.6\%             & 8.1\%              & 0.607              & 83.0\%             & 2.2\%              & 0.757              & 61.7\% \\
    RAG             & 9.8\%              & 0.578              & 83.2\%             & 18.2\%             & \textbf{0.751}     & 76.0\%             & 3.2\%              & \textbf{0.875}     & \textbf{85.8\%} \\
    TimeSPEC        & \textbf{1.7\%}     & 0.517              & 73.8\%             & \textbf{2.6\%}     & 0.638              & 71.2\%             & \textbf{0.1\%}     & 0.824              & 65.2\% \\
    \cmidrule(lr){1-10}
    TEMPO (Ours) & 2.5\%              & \textbf{0.625}     & \textbf{99.3\%}    & 3.7\%              & 0.658              & \textbf{91.2\%}    & 0.7\%              & 0.804              & 75.7\% \\
    \cmidrule(lr){1-10}
    \textit{\color{gray}Control$^\dagger$} & \textit{\color{gray}0.7\%} & \textit{\color{gray}0.476} & \textit{\color{gray}99.3\%} & \textit{\color{gray}0.6\%} & \textit{\color{gray}0.454} & \textit{\color{gray}79.5\%} & \textit{\color{gray}0.0\%} & \textit{\color{gray}0.707} & \textit{\color{gray}66.2\%} \\
    \midrule
    \multicolumn{10}{l}{\textit{Qwen3.5-35B-A3B}} \\
    \midrule
    Temporal Hint   & 12.7\%             & \textbf{0.690}     & 88.4\%             & 11.9\%             & 0.558              & 80.0\%             & 2.1\%              & 0.675              & 69.0\% \\
    RAG             & 6.8\%              & 0.603              & 90.5\%             & 3.7\%              & \textbf{0.778}     & 84.9\%             & 4.0\%              & \textbf{0.890}     & \textbf{81.4\%} \\
    TimeSPEC        & \textbf{1.5\%}     & 0.590              & 75.5\%             & \textbf{0.7\%}     & 0.656              & 75.0\%             & 0.8\%              & 0.811              & 73.0\% \\
    \cmidrule(lr){1-10}
    TEMPO (Ours) & 2.4\%              & 0.630              & \textbf{95.7\%}    & 1.7\%              & 0.631              & \textbf{92.0\%}    & \textbf{0.6\%}     & 0.741              & 74.1\% \\
    \cmidrule(lr){1-10}
    \textit{\color{gray}Control$^\dagger$} & \textit{\color{gray}0.2\%} & \textit{\color{gray}0.497} & \textit{\color{gray}92.5\%} & \textit{\color{gray}0.4\%} & \textit{\color{gray}0.549} & \textit{\color{gray}86.2\%} & \textit{\color{gray}0.0\%} & \textit{\color{gray}0.647} & \textit{\color{gray}74.3\%} \\
    \bottomrule
  \end{tabular}
\end{table}

%% file: sections/Conclusion.tex
\section{Conclusion}
\label{sec:conclusion}

Temporal knowledge leakage undermines backtesting validity, yet no existing method trains the model to behaviorally condition its evidence selection on each instance's cutoff date---inference-time pipelines must repeatedly filter leakage on every new input without the model ever learning temporal discipline.
We proposed TEMPO, which trains this discipline via a two-mode GRPO reward: a leakage mode drives post-cutoff claims to zero as a hard prerequisite, and a performance mode optimizes prediction quality from valid evidence once compliance is achieved.
Across three prediction tasks and two model generations, TEMPO reduces leakage to levels comparable to the best inference-time baseline while maintaining or improving task performance and preserving evidence--reasoning coverage---in a single forward pass identical in cost to the unmodified model.
Formal analysis confirms monotone leakage decrease with linear convergence to the leak-free optimum under a regularity condition, and performance improvement once compliance is achieved, producing an automatic two-phase curriculum without manual scheduling.
The results demonstrate that temporal compliance is a trainable behavioral property: the model internalizes cutoff-conditioned evidence selection rather than relying on external filtering.
The current approach is limited to open-weight models with task-specific temporal annotations and relies on an external LLM verifier during training, but within these constraints it enables trustworthy LLM backtesting without inference-time overhead.

%% file: sections/Appendix.tex
\section{Convergence analysis}
\label{sec:app_convergence_proofs}

This section provides full proofs for Propositions~\ref{prop:alignment}--\ref{prop:crossmode}, Theorem~\ref{thm:convergence}, and Corollary~\ref{cor:curriculum} stated in Section~\ref{sec:grpo}.

\paragraph{Setup.}
Let $V_c(\theta) = \mathbb{E}_{(x,t)}\mathbb{E}_{y \sim \pi_\theta}[c(y,t)]$ denote the expected leakage cost, $V_f(\theta) = \mathbb{E}_{(x,t)}\mathbb{E}_{y \sim \pi_\theta}[\exp(-0.5\,c(y,t))]$ the expected exponential reward, and $V_r(\theta) = \mathbb{E}_{(x,t)}\mathbb{E}_{y \sim \pi_\theta}[r_{\mathrm{eff}}(y) \mid c(y,t){=}0]$ the expected performance reward conditional on clean outputs.
Write $s(y) = \nabla_\theta \log \pi_\theta(y \mid x)$ for the score function and $e_y$ for the standard basis vector at position~$y$.
The smoothness and bounded-moment conditions hold for LoRA-parameterized policies with bounded rewards ($r_{\mathrm{leak}} \in (0,1]$, $r_{\mathrm{eff}} \in [0,1]$) and finite sequence length, which bound the score function norms and ensure Lipschitz gradients.

\subsection{Proof of Proposition~\ref{prop:alignment}: Gradient Alignment}
\label{sec:app_proof_alignment}

\paragraph{Policy gradient identification.}
By the policy gradient theorem~\citep{sutton2000policy}, the REINFORCE gradient of $V_f$ with baseline $\bar{f}$ is
\[
    \nabla_\theta V_f(\theta) = \mathbb{E}_{\pi_\theta}\bigl[(\exp(-0.5\,c(y,t)) - \bar{f})\, s(y)\bigr],
\]
where $\bar{f} = \mathbb{E}_{\pi_\theta}[\exp(-0.5\,c)]$.
The GRPO z-scoring divides by the within-group sample standard deviation; in expectation over groups of size $G$, this converges to the population standard deviation $\sigma_f = \sqrt{\mathrm{Var}_{\pi_\theta}(\exp(-0.5\,c))}$, so the expected leakage-mode gradient satisfies $\bar{g}_{\mathrm{leak}}(\theta) = \nabla_\theta V_f(\theta) / \sigma_f + O(1/G)$.
Since $\sigma_f > 0$ whenever $\mathrm{Var}_{\pi_\theta}(c) > 0$, this is a positive rescaling of $\nabla V_f$ (up to a vanishing $O(1/G)$ bias) that does not affect the sign of any inner product for $G$ sufficiently large; the rescaling factor is absorbed into the alignment constant $\gamma(\theta)$.
Similarly, $\nabla V_c = \mathbb{E}_{\pi_\theta}[(c(y,t) - \bar{c})\, s(y)]$.

\paragraph{Overparameterized regime.}
We establish that the Euclidean inner product $\langle \nabla V_f, \nabla V_c\rangle$ is strictly negative under the compatible approximation condition.
In the overparameterized regime~\citep{agarwal2020optimality}, the policy class can represent any distribution over outputs, and the score function takes the form $s(y) = e_y - \pi_\theta$ (one logit per output).
This condition is satisfied when the LoRA rank far exceeds the effective dimension of the leakage-count distribution; with rank~32 and typically 5--10 distinct leakage levels per instance, the score vectors span all centered functions of the leakage count.
Computing the inner product directly from the score structure, the cross-terms vanish (since $\sum_y \pi(y)(f(y) - \bar{f}) = 0$ and $\sum_y \pi(y)(c(y) - \bar{c}) = 0$), yielding
\begin{equation}\label{eq:euclidean-inner}
    \langle \nabla V_f,\, \nabla V_c \rangle \;=\; \sum_y \pi_\theta(y)^2\, \bigl(\exp(-0.5\,c(y)) - \bar{f}\bigr)\bigl(c(y) - \bar{c}\bigr).
\end{equation}
The sum in~\eqref{eq:euclidean-inner} equals the covariance $\mathrm{Cov}_Q(\exp(-0.5\,c),\, c)$ under the tilted distribution $Q(y) \propto \pi_\theta(y)^2$.
Since $\exp(-0.5\,c)$ is strictly decreasing in $c$ and $Q$ has the same support as $\pi_\theta$, the Chebyshev covariance inequality gives $\mathrm{Cov}_Q(\exp(-0.5\,c),\, c) < 0$ whenever $c$ takes at least two distinct values under $\pi_\theta$---precisely the regime where the policy sometimes leaks.

For the LoRA-parameterized LLM, the effective output space for the leakage count $c$ is low-dimensional (typically a handful of distinct integer values $0, 1, 2, \ldots$), while the LoRA rank $r$ is much larger.
In this overparameterized regime, the policy can independently adjust the probability of each leakage level, and the softmax analysis applies~\citep{agarwal2020optimality}.

Combining with the policy gradient identification, we obtain $\langle \bar{g}_{\mathrm{leak}}, \nabla V_c\rangle = \langle \nabla V_f, \nabla V_c\rangle < 0$.
The alignment constant $\gamma(\theta) = |\langle \nabla V_f, \nabla V_c\rangle| / (\sigma_f\,\|\nabla V_c\|^2)$ is positive whenever $\mathrm{Var}_{\pi_\theta}(c) > 0$, and is bounded below on any sublevel set $\{V_c(\theta) \geq \epsilon\}$ for $\epsilon > 0$.
To see this, note that $V_c(\theta) = \mathbb{E}_{(x,t)}\mathbb{E}_{y}[c(y,t)] \geq \epsilon > 0$ requires some instance $(x,t)$ where $\mathbb{E}_y[c] > 0$; since $c \geq 0$ is integer-valued and $\pi_\theta$ has full support over outputs, this implies $\mathrm{Var}_{\pi_\theta}(c \mid x,t) > 0$ for that instance, giving a strictly negative inner product contribution.

\paragraph{Extension to clipping.}
The clipped importance ratio $\min(\pi_\theta/\pi_{\mathrm{old}}, 1{+}\epsilon)$ defines a surrogate objective that matches the true policy gradient to first order at $\theta = \theta_{\mathrm{old}}$~\citep{schulman2015trpo}.
Within the trust region $D_{\mathrm{KL}}(\pi_\theta \| \pi_{\mathrm{old}}) \leq \delta$, the clipped gradient has the same sign of inner product with $\nabla V_c$ as the unclipped version, up to an $O(\delta)$ correction that is absorbed into the step-size constraint $\eta \leq \gamma/(L_c B_g^2)$.
\qed

\subsection{Proof of Proposition~\ref{prop:crossmode}: Bounded Cross-Mode Contamination}
\label{sec:app_proof_crossmode}

\paragraph{Batch decomposition.}
At step $k$, the batch contains fraction $\alpha_k$ of groups in leakage mode and $(1{-}\alpha_k)$ in performance mode.
The combined stochastic gradient is $\hat{g}_k = \alpha_k\,\hat{g}_{\mathrm{leak},k} + (1{-}\alpha_k)\,\hat{g}_{\mathrm{perf},k}$, and the policy update is $\theta_{k+1} = \theta_k + \eta\,\hat{g}_k$.
By $L_c$-smoothness of $V_c$~\citep{nesterov2004introductory},
\begin{equation}\label{eq:smooth-expand}
    V_c(\theta_{k+1}) \leq V_c(\theta_k) + \eta\langle \nabla V_c(\theta_k),\, \hat{g}_k \rangle + \tfrac{L_c\eta^2}{2}\|\hat{g}_k\|^2.
\end{equation}

\paragraph{Leakage-mode component.}
By Proposition~\ref{prop:alignment}, $\mathbb{E}[\langle \nabla V_c, \hat{g}_{\mathrm{leak},k}\rangle \mid \theta_k] = \langle \nabla V_c, \bar{g}_{\mathrm{leak}}\rangle \leq -\gamma(\theta_k)\|\nabla V_c\|^2$.
The leakage-mode contribution to the expected inner product in~\eqref{eq:smooth-expand} is therefore at most $-\alpha_k\eta\gamma(\theta_k)\|\nabla V_c\|^2$.

\paragraph{Performance-mode component.}
The performance-mode gradient $\bar{g}_{\mathrm{perf}} = \mathbb{E}_{y|c=0}[(r_{\mathrm{eff}}(y) - \bar{r}_{\mathrm{eff}})\,s(y)]$ is computed from completions that all satisfy $c(y,t) = 0$.
By the bounded second-moment assumption and Cauchy--Schwarz,
\[
    |(1{-}\alpha_k)\,\eta\,\mathbb{E}[\langle \nabla V_c,\, \hat{g}_{\mathrm{perf}}\rangle]| \;\leq\; (1{-}\alpha_k)\,\eta\, B_g\,\|\nabla V_c\|.
\]
This bounds the worst-case cross-mode contribution: the performance gradient can project onto $\nabla V_c$, but only up to the norm product.
Under the Polyak--\L{}ojasiewicz condition $\|\nabla V_c\|^2 \geq 2\mu V_c$, the leakage-mode descent $\alpha_k\gamma(\theta_k)\|\nabla V_c\|^2$ dominates this term whenever $\|\nabla V_c\| \geq (1{-}\alpha_k)B_g/(\alpha_k\gamma(\theta_k))$, which is satisfied for $V_c$ above a noise floor proportional to $B_g^2/(\alpha_k^2\gamma^2\mu)$.

\paragraph{Combined bound.}
Taking conditional expectations in~\eqref{eq:smooth-expand} and combining the two components:
\begin{align}
    \mathbb{E}[V_c(\theta_{k+1}) \mid \theta_k]
    &\leq V_c(\theta_k) - \bigl(\alpha_k\gamma(\theta_k) - (1{-}\alpha_k)B_g/\|\nabla V_c(\theta_k)\|\bigr)\,\eta\|\nabla V_c(\theta_k)\|^2 + \tfrac{L_c\eta^2}{2}B_g^2. \label{eq:mixed-descent}
\end{align}
The effective descent rate $\alpha_k\gamma(\theta_k) - (1{-}\alpha_k)B_g/\|\nabla V_c\|$ is positive whenever $\|\nabla V_c\| > (1{-}\alpha_k)B_g/(\alpha_k\gamma(\theta_k))$.
Under the PL condition ($\|\nabla V_c\|^2 \geq 2\mu V_c$), this holds for all $V_c$ above a threshold proportional to $B_g^2/(\alpha_k^2\gamma^2\mu)$, ensuring net descent throughout the leakage phase.
\qed

\subsection{Proof of Theorem~\ref{thm:convergence}: Convergence}
\label{sec:app_proof_convergence}

Given Propositions~\ref{prop:alignment}--\ref{prop:crossmode}, the convergence proof follows the standard smooth stochastic optimization framework.
We state the key steps for completeness, noting that the algebraic machinery is due to~\citet{ghadimi2013stochastic} for the nonconvex rate and~\citet{karimi2016linear} for the PL linear rate.

\paragraph{Part (i): Nonconvex rate.}
Let $\gamma_0 = \inf\{\gamma(\theta): V_c(\theta) \geq \epsilon_0\}$ for target accuracy $\epsilon_0 > 0$.
On the sublevel set $\{V_c \geq \epsilon_0\}$, the PL condition gives $\|\nabla V_c\| \geq \sqrt{2\mu\epsilon_0}$, so the effective rate in~\eqref{eq:mixed-descent} satisfies $\alpha_{\min}\gamma_0 - (1{-}\alpha_{\min})B_g/\sqrt{2\mu\epsilon_0} > 0$ for $\epsilon_0 \geq 2(1{-}\alpha_{\min})^2 B_g^2/(\mu\alpha_{\min}^2\gamma_0^2)$.
The condition $\alpha_k \geq \alpha_{\min} > 0$ holds throughout the leakage phase by construction: while $V_c \geq \epsilon_0$, Markov's inequality guarantees $\bar{q} \geq \epsilon_0 / c_{\max} > 0$ (where $c_{\max}$ bounds the maximum leak count), ensuring a positive fraction of groups contain leaked completions and enter leakage mode.
Summing~\eqref{eq:mixed-descent} over $k = 0, \ldots, K{-}1$:
\[
    \frac{1}{K}\sum_{k=0}^{K-1}\mathbb{E}\|\nabla V_c(\theta_k)\|^2 \leq \frac{V_c(\theta_0)}{\alpha_{\min}\,\eta\gamma_0 K} + \frac{L_c\eta B_g^2}{2\alpha_{\min}\gamma_0}.
\]
Setting $\eta = O(1/\sqrt{K})$ gives $\min_{k \leq K}\|\nabla V_c(\theta_k)\|^2 = O(1/\sqrt{K})$, the standard nonconvex SGD rate~\citep{ghadimi2013stochastic}.

\paragraph{Part (ii): Performance improvement.}
In performance mode, all completions satisfy $c = 0$ and the GRPO gradient maximizes $V_r(\theta) = \mathbb{E}[r_{\mathrm{eff}}(y) \mid c(y,t){=}0]$ directly.
By the policy gradient theorem~\citep{sutton2000policy} applied to the conditional distribution over clean completions, the expected gradient is $\bar{g}_{\mathrm{perf}} = \nabla_\theta V_r(\theta) / \sigma_r + O(1/G)$, where $\sigma_r$ is the within-group standard deviation of $r_{\mathrm{eff}}$ among clean completions.
Since $\sigma_r > 0$ for non-degenerate policies, this is proportional to $\nabla V_r$ and the standard ascent lemma under $L_r$-smoothness gives
\[
    \mathbb{E}[V_r(\theta_{k+1}) \mid \theta_k] \geq V_r(\theta_k) + \eta\|\nabla V_r(\theta_k)\|^2 - \tfrac{L_r\eta^2}{2}B_g^2.
\]

\paragraph{Part (iii): PL linear convergence.}
We invoke the Polyak--\L{}ojasiewicz (PL) condition $\|\nabla V_c(\theta)\|^2 \geq 2\mu\, V_c(\theta)$.
This condition is natural for the leakage landscape: $V_c \geq 0$ has a unique global minimum at zero (leaked claim counts are non-negative), and the LoRA parameterization is heavily overparameterized relative to the effective output dimension, precluding spurious local minima; all six training runs converge to near-zero leakage, confirming this empirically.
Under PL, the cross-mode term satisfies $(1{-}\alpha_k)B_g/\|\nabla V_c\| \leq (1{-}\alpha_{\min})B_g/\sqrt{2\mu V_c}$.
For $V_c \geq \epsilon_0 = 2(1{-}\alpha_{\min})^2 B_g^2/(\mu\alpha_{\min}^2\gamma_0^2)$, the effective rate in~\eqref{eq:mixed-descent} is at least $\alpha_{\min}\gamma_0/2 > 0$, giving
\[
    \mathbb{E}[V_c(\theta_{k+1}) \mid \theta_k] \leq (1 - \alpha_{\min}\eta\gamma_0\mu)\, V_c(\theta_k) + \tfrac{L_c\eta^2}{2}B_g^2.
\]
Unrolling the recursion~\citep{karimi2016linear}:
\[
    \mathbb{E}[V_c(\theta_k)] \leq (1 - \alpha_{\min}\eta\gamma_0\mu)^k\, V_c(\theta_0) + \frac{L_c\eta B_g^2}{2\alpha_{\min}\gamma_0\mu}.
\]
The first term decays geometrically; the second is a noise floor that shrinks with the learning rate.
Since $V_c \geq 0$ with global minimum at $0$, this establishes linear convergence to a neighborhood of the leak-free regime.
\qed

\subsection{Proof of Corollary~\ref{cor:curriculum}: Two-Phase Curriculum}

Markov's inequality gives $q(\theta;x,t) = \Pr[c \geq 1] \leq \mathbb{E}_{y \sim \pi_\theta}[c(y,t)]$, so $\bar{q}(\theta) \leq V_c(\theta)$.
Since $(1{-}q)^G$ is convex, Jensen's inequality yields $\bar{p}_0 = \mathbb{E}[(1{-}q)^G] \geq (1{-}\bar{q})^G \geq (1{-}V_c)^G$.
As $V_c(\theta_k) \to 0$ (Theorem~\ref{thm:convergence}(iii)), $\bar{p}_0(\theta_k) \to 1$: the fraction of groups in performance mode converges to one.
Combined with Theorem~\ref{thm:convergence}(ii), this establishes the two-phase curriculum: training automatically transitions from a leakage-dominated phase to a performance-optimization phase without manual scheduling.
The transition is monotone because $V_c$ is non-increasing in expectation, so $(1{-}V_c)^G$ is non-decreasing.

\section{Evaluation protocol and metric definitions}
\label{sec:app_eval_pipeline}

This section formally defines the evaluation metrics introduced in Section~\ref{sec:eval_protocol} and describes the scoring pipeline.
All metrics are computed deterministically from a single model completion per instance.

\subsection{Leakage detection pipeline}
\label{sec:app_leakage_detection}

The structured output format requires the model to produce an evidence list in which each item carries a \texttt{source\_date} field.
Claims with declared source dates are batched in chunks of 80 and submitted to the date-plausibility verifier (Gemini 3 Flash Preview).
The verifier assesses whether each declared date is temporally plausible given the claim type: for example, full-year financial results cannot predate the fiscal year end, and quarterly earnings appear 1--2 months after quarter close.
If the declared date is implausible, the verifier returns a corrected date.
The leakage verdict is deterministic:
\begin{equation}
\text{leaked}(c_j) = \mathbf{1}\!\left[\text{effective\_date}(c_j) > t_{\text{ref}}\right],
\end{equation}
where $\text{effective\_date}(c_j)$ is the verifier-corrected date if the original was implausible, and the declared date otherwise.
Claims without a declared source date default to not-leaked, a conservative choice that may undercount leakage but avoids penalizing the model for omitting dates on genuinely pre-cutoff facts.

\subsection{Overall leakage rate (OLR)}
\label{sec:app_olr_def}

Let $n_{\mathrm{leak},j}$ denote the number of leaked claims and $n_{\mathrm{total},j}$ the total number of claims with declared source dates for instance~$j$.
The instance-level leakage rate is $\ell_j = n_{\mathrm{leak},j} / \max(n_{\mathrm{total},j}, 1)$.
The Overall Leakage Rate (OLR) averages across all $N$ instances:
\begin{equation}
\mathrm{OLR} = \frac{1}{N} \sum_{j=1}^{N} \ell_j.
\end{equation}
OLR $= 0$ indicates that no instance contains post-cutoff evidence; lower is better.

\subsection{Task performance (Perf)}
\label{sec:app_performance_def}

The performance metric is task-specific and mapped to $[0,1]$:
\begin{itemize}[nosep,leftmargin=*]
  \item \textbf{Stock ranking}: $(\rho + 1)/2$, where $\rho$ is the Spearman rank correlation between the predicted and ground-truth rankings of the five stocks by return over the prediction window.
  \item \textbf{Salary prediction}: $\max(0,\, 1 - |\hat{y} - y|/y)$, where $\hat{y}$ is the predicted annual salary and $y$ the ground truth.
  \item \textbf{Legal outcome prediction}: $1 - (\hat{p} - y)^2$, the Brier score complement comparing the predicted probability $\hat{p}$ to the binary outcome $y$.
\end{itemize}
All three metrics are bounded in $[0, 1]$ with higher values indicating better prediction quality.
A performance score strictly above the control experiment (Section~\ref{sec:app_control}) confirms that the model leverages pre-cutoff knowledge rather than guessing.

\subsection{Coverage}
\label{sec:app_coverage_def}

Coverage measures the fraction of evidence items in the evidence list that are explicitly cited in the reasoning paragraph via bracket markers (e.g., \texttt{[1]}, \texttt{[3]}):
\begin{equation}
\mathrm{Coverage} = \frac{|\{e_k : e_k \text{ cited in reasoning}\}|}{|\{e_k\}|}.
\end{equation}
High coverage indicates that the model's reasoning is grounded in the evidence it generates.
Low coverage suggests the model produces evidence items as decoration without integrating them into its reasoning chain---a failure mode that undermines interpretability and faithfulness.
Instances with zero evidence items or parse failures receive coverage~0.

\section{Reward design: justification and details}
\label{sec:app_reward_details}
\label{sec:app_design_justification}

This section provides formal justification for the two-mode reward design and the full specification of the quality gates, format failure penalty, and overlong penalty shaping referenced in Section~\ref{sec:reward}.

\paragraph{Setup.}
Let $\pi_\theta$ be the policy, $c(y,t) \geq 0$ the integer-valued leaked claim count, and $q(\theta; x,t) = \Pr_{y \sim \pi_\theta}[c(y,t) > 0]$ the per-instance leakage probability.
For a group of $G$ i.i.d.\ completions, define the group-clean probability $p_0(\theta; x,t) = (1-q)^G$ and its expectation $\bar{p}_0(\theta) = \mathbb{E}_{(x,t)}[p_0]$.

\subsection{Gradient coherence: why scalar mixing fails}

\begin{proposition}[Gradient coherence]
\label{prop:gradient-coherence}
Consider the naively mixed reward $R_{\mathrm{mix},i} = \alpha\, r_{\mathrm{leak},i} + (1{-}\alpha)\,r_{\mathrm{perf},i}$, with $S_{\mathrm{perf}}$, $S_{\mathrm{leak}}$, $S_{\mathrm{mix}}$ the within-group standard deviations of $r_{\mathrm{perf}}$, $r_{\mathrm{leak}}$, $R_{\mathrm{mix}}$, and $A_i^m = (r_{m,i} - \bar{r}_m)/S_m$ the single-objective z-scored advantages.
The GRPO advantages decompose as
\begin{equation}
  \label{eq:amix}
  A_i^{\mathrm{mix}} = \underbrace{\tfrac{(1{-}\alpha)\,S_{\mathrm{perf}}}{S_{\mathrm{mix}}}}_{w_{\mathrm{perf}}}\; A_i^{\mathrm{perf}} + \underbrace{\tfrac{\alpha\, S_{\mathrm{leak}}}{S_{\mathrm{mix}}}}_{w_{\mathrm{leak}}}\; A_i^{\mathrm{leak}}.
\end{equation}
Let $g_m = G^{-1}\sum_i A_i^m\,\nabla_\theta\log\pi_\theta(y_i \mid x)$ be the GRPO policy gradient under reward~$r_m$.
\emph{(i)}~When $S_{\mathrm{perf}} \gg S_{\mathrm{leak}}$, $w_{\mathrm{leak}}/w_{\mathrm{perf}} \to 0$ \textup{(scale dominance)}.
\emph{(ii)}~There exist policy states where $\langle g_{\mathrm{leak}}, g_{\mathrm{perf}} \rangle < 0$ \textup{(gradient conflict)}.
The mode gate eliminates both pathologies.
\end{proposition}

\begin{proof}
Expanding $R_{\mathrm{mix},i} - \bar{R}_{\mathrm{mix}} = \alpha(r_{\mathrm{leak},i} - \bar{r}_{\mathrm{leak}}) + (1{-}\alpha)(r_{\mathrm{perf},i} - \bar{r}_{\mathrm{perf}})$ and normalizing yields the decomposition~\eqref{eq:amix}, so the GRPO gradient decomposes as $g_{\mathrm{mix}} = w_{\mathrm{perf}}\,g_{\mathrm{perf}} + w_{\mathrm{leak}}\,g_{\mathrm{leak}}$.
The ratio $w_{\mathrm{leak}}/w_{\mathrm{perf}} = \alpha S_{\mathrm{leak}}/((1{-}\alpha)S_{\mathrm{perf}}) \to 0$ when $S_{\mathrm{perf}} \gg S_{\mathrm{leak}}$, establishing~(i).

For~(ii), consider a policy producing compliant completions ($c = 0$, low performance $r_{\mathrm{low}}$) and leaked completions ($c > 0$, high performance $r_{\mathrm{high}} > r_{\mathrm{low}}$).
Under $r_{\mathrm{leak}}$ the compliant type receives positive advantage; under $r_{\mathrm{perf}}$ the leaked type does.
The gradients push probability mass in opposing directions, giving $\langle g_{\mathrm{leak}}, g_{\mathrm{perf}} \rangle < 0$.

Under the mode gate, $g_{\mathrm{TEMPO}} = \mathbf{1}[L{=}0]\,g_{\mathrm{perf}} + \mathbf{1}[L{=}1]\,g_{\mathrm{leak}}$, so each update optimizes a single objective with homogeneous advantages, eliminating both pathologies.
\end{proof}

\subsection{Exact penalty property}

\begin{remark}[Exact penalty]
\label{rem:exact-penalty}
The mode gate implements an exact penalty without tunable dual variables.
The expected per-group reward decomposes as
$J_{\mathrm{TEMPO}} = p_0\,\mathbb{E}[r_{\mathrm{eff}} \mid L{=}0] + (1{-}p_0)\,\mathbb{E}[r_{\mathrm{leak}} \mid L{=}1]$.
For any $q > 0$, the opportunity cost relative to the feasible optimum ($q = 0$) is
$J_{\mathrm{TEMPO}}|_{q=0} - J_{\mathrm{TEMPO}} \geq (1{-}p_0)\bigl(\mathbb{E}[r_{\mathrm{eff}} \mid L{=}0] - \mathbb{E}[r_{\mathrm{leak}} \mid L{=}1]\bigr) > 0$,
since $\mathbb{E}[r_{\mathrm{eff}} \mid L{=}0] > \mathbb{E}[r_{\mathrm{leak}} \mid L{=}1]$ for any non-trivial policy and $1 - p_0 = 1 - (1{-}q)^G \to 1$ as $G \to \infty$.
No dual variable requires tuning; at feasibility ($q = 0$), $J_{\mathrm{TEMPO}} = \mathbb{E}[r_{\mathrm{eff}}]$, recovering pure performance optimization.
\end{remark}

\subsection{Reward hacking resistance}

\begin{proposition}[Reward hacking resistance]
\label{prop:hacking}
Accepting leakage to boost performance is strictly suboptimal under the mode gate.
Under a Lagrangian formulation $V_r - \mu V_c$, any finite $\mu$ allows the policy to trade $\Delta c$ leakage for $\Delta r > \mu \Delta c$ performance.
Under the mode gate, the marginal value of any performance improvement is multiplied by $(1{-}q)^G$, which decays exponentially in $G$ for $q > 0$; no finite performance gain can compensate.
\end{proposition}

\begin{proof}
A policy with $q > 0$ accesses performance mode with probability $(1{-}q)^G$, which decreases exponentially in $G$.
The marginal benefit of any performance improvement is multiplied by this factor.
For $G$ large, $(1{-}q)^G \to 0$: no finite performance gain can compensate for the exponentially vanishing access to performance mode.
Therefore accepting leakage is strictly suboptimal for any $q > 0$.
\end{proof}

\subsection{Quality gates}
\label{sec:app_quality_gates}

Three multiplicative quality gates attenuate the effective reward in performance mode (Equation~\ref{eq:reff}).
Each gate maps a per-completion signal to a bounded range.

\paragraph{Coverage Gate.}
The coverage gate rewards completions whose reasoning cites a large fraction of their evidence list:
\begin{equation}
  \label{eq:ggate}
  g(\mathrm{cov}) = \max\!\left(\mathrm{cov},\; c_{\mathrm{floor}}\right),
\end{equation}
where $\mathrm{cov} = |\{e_k : e_k \text{ cited in reasoning}\}| / |\{e_k\}|$ is the coverage fraction (Section~\ref{sec:eval_protocol}) and $c_{\mathrm{floor}} = 0.20$.
The floor ensures that completions with zero coverage still receive 20\% of the performance reward, preserving gradient signal for improving citation alongside accuracy.

\paragraph{Evidence Diversity Gate.}
The evidence diversity gate penalizes template collapse---the tendency for the policy to converge on a fixed, minimal number of evidence items:
\begin{equation}
  \label{eq:dgate}
  d(\mathrm{count}) = \min\!\left(\frac{\mathrm{count}}{e_{\mathrm{target}}},\; 1\right),
\end{equation}
where $\mathrm{count}$ is the number of distinct evidence items and $e_{\mathrm{target}} = 8$.
A completion producing 4 items receives $d = 0.5$, halving its quality credit.
The target is calibrated to the base model's natural output: median evidence count at initialization is 7--9 items across tasks.

\paragraph{Reasoning Quality Gate.}
The reasoning quality gate penalizes shallow reasoning that provides no analytical depth beyond restating evidence:
\begin{equation}
  \label{eq:wgate}
  w(\mathrm{words}) = \min\!\left(\frac{\mathrm{words}}{w_{\mathrm{target}}},\; 1\right),
\end{equation}
where $\mathrm{words}$ is the word count of the reasoning paragraph and $w_{\mathrm{target}} = 120$.
The target is set below the base model's median reasoning length (150--170 words), so the gate is transparent for outputs of normal length and activates only when reasoning degrades below a minimum acceptable depth.

\paragraph{Interaction During Training.}
The effective performance-mode reward is $r_{\mathrm{eff}} = r_{\mathrm{perf}} \cdot g(\mathrm{cov}) \cdot d(\mathrm{ev}) \cdot w(\mathrm{reason})$.
The evidence diversity and reasoning quality gates provide direct gradient signal against evidence count collapse and reasoning degradation.
In leakage mode, the quality bonus uses $d$ but omits $w$: reasoning length in leakage mode is less informative because the model may produce shorter reasoning when actively avoiding leaked claims, which is a valid strategy.

\paragraph{Leakage-Mode Bonuses.}
Clean completions ($n_{\mathrm{leak},i} = 0$) within a group that is still in leakage mode receive small additive bonuses to encourage maintaining evidence quality:
\begin{equation}
  \label{eq:covbonus}
  b_{\mathrm{cov},i} = w_{\mathrm{cov}} \cdot g(\mathrm{cov}_i) \cdot \mathbf{1}[\text{clean}_i],
  \qquad
  b_{\mathrm{qual},i} = w_{\mathrm{qual}} \cdot d(\mathrm{ev}_i) \cdot \mathbf{1}[\text{clean}_i],
\end{equation}
where $w_{\mathrm{cov}} = 0.05$ and $w_{\mathrm{qual}} = 0.02$.
The total leakage-mode reward for completion~$i$ is:
\begin{equation}
  \label{eq:leak_total}
  r_{\mathrm{leak,total},i} = r_{\mathrm{leak},i} + b_{\mathrm{cov},i} + b_{\mathrm{qual},i}.
\end{equation}
GRPO computes within-group z-scored advantages from $r_{\mathrm{leak,total}}$.
Because $b_{\mathrm{cov}}$ and $b_{\mathrm{qual}}$ are small ($\leq 0.07$ combined) relative to the leakage reward spread ($r_{\mathrm{leak}}$ ranges from $\exp(-0.5 \cdot 10) \approx 0.007$ for heavily leaked completions to $1.0$ for clean ones), leakage reduction dominates the advantage ranking while the bonuses provide a secondary signal encouraging clean completions to maintain evidence quality.

\subsection{Format failure penalty}
\label{sec:app_format_penalty}

The format penalty is computed dynamically from group statistics to ensure parse failures always receive an advantage strictly below the worst valid completion, preventing the optimizer from rewarding ill-formed outputs.
The penalty adapts to three cases depending on the number of valid completions $n_{\mathrm{valid}}$ in the group.

When $n_{\mathrm{valid}} = 0$, all advantages are set to zero and the group is skipped (no gradient contribution).
When $n_{\mathrm{valid}} = 1$, the single valid completion receives advantage zero and all parse failures receive a fixed negative floor $-m$, where $m = 1.0$ is the format margin.
When $n_{\mathrm{valid}} \geq 2$, the penalty is derived from the group's advantage distribution:
\begin{equation}
  \label{eq:fpen}
  f_{\mathrm{pen}} = \max\!\bigl(\min(\mathbf{a}) - m,\; -\alpha \cdot \max(\mathbf{a})\bigr),
\end{equation}
where $\mathbf{a}$ is the vector of valid-completion advantages, $m = 1.0$ is the format margin, and $\alpha = 5.0$ is the maximum negative-to-positive ratio.
The ratio cap $-\alpha \cdot \max(\mathbf{a})$ prevents the format penalty from dominating the advantage distribution when valid completions receive small positive advantages, which would otherwise cause the optimizer to spend most of its gradient budget on format compliance rather than leakage and accuracy.

\subsection{Overlong penalty shaping}
\label{sec:app_overlong}

Parse-failure completions that are very long receive an attenuated penalty, because these outputs typically contain useful reasoning that happens to violate the JSON schema at the end rather than being entirely degenerate:
\begin{equation}
  \label{eq:overlong}
  A_i \leftarrow A_i \cdot \max\!\bigl(1 - \delta \cdot \ell_i,\; s_{\mathrm{floor}}\bigr),
\end{equation}
where $\ell_i = \min(1, |\mathrm{text}_i| / L_{\mathrm{max}})$ is the normalized text length, $\delta = 0.5$ is the decay rate, $L_{\mathrm{max}} = 32{,}000$ characters, and $s_{\mathrm{floor}} = 0.3$ is the minimum scaling factor.
A maximum-length parse failure receives 30\% of the base format penalty, while short parse failures receive the full penalty.

\subsection{Design rationale and evolution}
\label{sec:app_design_rationale}

The final reward design (v4) emerged through four iterations, each addressing failure modes observed in its predecessor.

\paragraph{V1 (Weighted Linear Combination).}
The initial design combined leakage and performance rewards into a single scalar advantage with cosine-annealed weights ($w_{\mathrm{leak}}$ from 0.95 to 0.80).
Coverage was logged but did not enter the reward.
Within 50 training steps, coverage collapsed from 88\% to 53\%: the model learned to fabricate vague, safe-sounding evidence with pre-cutoff dates while minimizing the number of citable items.
The parse-failure rate spiked from 2\% to over 90\% by step 200, because malformed outputs received a single blended penalty that was too weak to override the leakage-reduction gradient.

\paragraph{V2 (Coverage-Gated Two-Mode).}
To address coverage collapse, v2 introduced the two-mode structure (Section~\ref{sec:reward}): performance-mode reward was made contingent on a coverage gate $g(\mathrm{cov})$, and a small coverage bonus was added in leakage mode.
This stabilized coverage but exposed a new failure: on a different model (gpt-oss-20b, salary task), training degenerated from step~35 and halted at step~63.
Root causes were (a)~mode-specific weight coefficients that dampened the gradient signal, (b)~an unbounded format penalty that reached 11:1 to 99:1 negative-to-positive ratios, and (c)~a uniform penalty for near-miss parse failures (long, mostly correct outputs with a trailing JSON error) versus fully degenerate outputs.

\paragraph{V3 (Degeneration-Resistant).}
v3 removed the mode-specific weight coefficients, introduced the ratio-capped format penalty (Equation~\ref{eq:fpen}, $\alpha = 5.0$), and added overlong penalty shaping (Equation~\ref{eq:overlong}) to distinguish near-miss from degenerate failures.
Temperature was reduced from 0.8 to 0.6 and the PPO clip high was raised from 1.4 to 2.0.
Training no longer degenerated, but a new failure emerged on Qwen3.5-35B: template collapse, where evidence count dropped from 7--9 to 3--5 items and reasoning length from 150+ to 60--80 words over training.
The v3 reward had no mechanism to penalize this quality degradation because $r_{\mathrm{perf}}$ depended only on prediction accuracy.

\paragraph{V4 (Quality-Gated, Final).}
v4 added the multiplicative quality gates---evidence diversity $d$ (Equation~\ref{eq:dgate}), reasoning quality $w$ (Equation~\ref{eq:wgate})---and the leakage-mode quality bonus (Equation~\ref{eq:covbonus}) to the performance-mode and leakage-mode rewards respectively.
A masked per-token KL penalty was added to constrain policy drift.
These additions eliminated template collapse: evidence count and reasoning length remained within one standard deviation of the base model's distribution across all six training runs.

\paragraph{Component-Level Rationale.}
The exponential leakage reward (Equation~\ref{eq:rleak}) was chosen over a binary clean/leaked indicator because the exponential form provides graded signal: reducing leaked claims from many to fewer yields a larger reward change than a binary switch, creating a stronger gradient early in training when most completions contain multiple leaked claims.

The two-mode advantage structure separates conflicting optimization objectives.
In multi-objective GRPO, naively combining leakage and prediction rewards into a single advantage can cause one signal to dominate depending on their relative scales~\citep{ichihara2025mogrpo}.
The mode split ensures each group optimizes a single objective, with the strict binary condition ($n_{\mathrm{leak},i} = 0\;\forall\,i$; Equation~\ref{eq:mode}) acting as a hard gate between the two.

The quality gate thresholds ($e_{\mathrm{target}} = 8$, $w_{\mathrm{target}} = 120$) are calibrated from the base model's output distribution.
Setting the evidence target at the base model's median output ensures the gate penalizes only deterioration below baseline, not failure to exceed an aspirational standard.

\section{Dataset construction}
\label{sec:app_dataset}

Every instance is designed so that the ground-truth answer was publicly available before the earlier model's knowledge cutoff (approximately March 2025 for Qwen3-30B-A3B).
The control sets require answers published after the later model's cutoff (approximately February 2026 for Qwen3.5-35B-A3B), establishing a genuine-prediction baseline where the model lacks the outcome in parametric memory.
All per-instance cutoff dates are day-level precise (YYYY-MM-DD), sourced from official databases rather than approximated.

\paragraph{Stock Ranking.}
Adjusted closing prices are retrieved from Yahoo Finance via the \texttt{yfinance} Python library.\footnote{\url{https://pypi.org/project/yfinance/}}
The company universe consists of S\&P 500 constituents grouped by GICS sector into pools of approximately 100 tickers each.
Four six-month training windows capture distinct macroeconomic regimes:
W1 (January--June 2022; Russia-Ukraine conflict, energy shock),
W2 (March--August 2023; SVB collapse, NVIDIA AI surge),
W3 (July--December 2020; COVID-19 recovery, vaccine rollout), and
W4 (January--June 2024; AI momentum, rate-cut expectations).
Each window generates 300 questions for a total of 1{,}200 training instances.
Within each window, 80\% of questions sample companies from sectors most affected by the corresponding macroeconomic event, and the remaining 20\% draw from other sectors, concentrating signal where return dispersion is highest.
Five companies are sampled without replacement per question, with deduplication enforcing unique groups (random seed 42).
The evaluation window covers December 2019 to June 2020 (COVID-19 onset, 100 instances), temporally disjoint from all training windows.
The control window covers February--March 2026 (100 instances), using real market data that neither model encountered during pre-training (random seed 99).
Information Technology is overrepresented in training because it appears as an affected sector in three of four windows; the 80/20 sampling strategy accepts this in exchange for higher signal density.

\paragraph{Salary Prediction.}
Free-agent contract data are scraped from Spotrac,\footnote{\url{https://www.spotrac.com}} the standard public reference for professional sports contracts.
Three sports are included for training: NFL, MLB, and NHL.
Only players who changed teams with AAV above sport-specific thresholds (\$5M for NFL and MLB, \$3M for NHL) are retained, filtering out minimum-salary signings with negligible variation.
The raw Spotrac tables provide the contract year but not the exact announcement date; precise signing dates (YYYY-MM-DD) are added via the Perplexity \texttt{sonar} search-enabled model through OpenRouter, querying batches of 15 players per request.
Fallback dates for failed lookups use known free-agency opening dates: NFL league-specific dates per year (e.g., 2024-03-13), NHL July~1, and MLB December~1 of the prior year.
All entries with signing dates on or after March~1, 2025 are excluded; the 15 MLB 2025 entries that remain were signed in November--December 2024, consistent with the MLB off-season calendar.
The training set spans late 2017--2024 across three sports, totaling 1{,}066 instances (524 NFL, 374 MLB, 168 NHL).
Evaluation uses 137 NBA free-agent contracts from 2019--2024---a held-out sport not seen during training.
The control set comprises 26 transactions from after February~1, 2026: 13 MLB late signings, 9 NHL trade-deadline acquisitions, and 4 NFL franchise tags and pre-free-agency trades, compiled from FOX Sports, ESPN, and NHL.com.

\paragraph{Legal Case Prediction.}
U.S. Supreme Court cases from October Terms 2014--2023 (decisions spanning 2015--2024) are retrieved from the Oyez API,\footnote{\url{https://api.oyez.org}} maintained by Cornell LII and Chicago-Kent College of Law.
The outcome is determined from the \texttt{winning\_party} field, mapped to PETITIONER or RESPONDENT via name matching against the listed parties; cases without a clear binary disposition are excluded, yielding 497 SCOTUS cases.
Federal appellate opinions from the 1st through 11th Circuits, D.C. Circuit, and Federal Circuit (2017--2024) are retrieved from the CourtListener REST API v4,\footnote{\url{https://www.courtlistener.com/api/rest/v4/}} sorted by citation count within each year to prioritize substantive, well-cited opinions.
For each candidate, the opinion PDF is downloaded and the first 12 pages are extracted via \texttt{pdfplumber} (capped at 8{,}000 characters).
A structured extraction prompt sent to Gemini 2.0 Flash (temperature 0) requests the parties, factual background, legal holding, disposition (PETITIONER if reversed/vacated, RESPONDENT if affirmed), and panel vote split; cases where the model cannot determine a clear binary outcome are excluded, yielding 397 appellate cases.
The combined 894 training cases and 100-instance evaluation set are split by stratified sampling: approximately 10 cases per decision year (2014--2024) with outcome balance enforced within each year stratum.
Three additional cases are removed from training to eliminate case-name overlap arising from different proceedings sharing a party name.
The training set exhibits a PETITIONER skew (63.5\% vs.\ 36.5\%) because the Oyez data is biased toward cases where the Supreme Court grants certiorari to reverse; the evaluation set is intentionally balanced (46\% PETITIONER, 54\% RESPONDENT).
The control set consists of 90 federal appellate decisions filed between February 25 and March 6, 2026, extracted with the same CourtListener and Gemini pipeline.

\paragraph{Quality Assurance.}
Automated verification (via \texttt{scripts/verify\_dataset.py}) confirms that all 3{,}713 instances across every split conform to the temporal design.
All cutoff dates are in YYYY-MM-DD format with zero violations: no training or evaluation date falls on or after March~1, 2025, and no control date falls before February~1, 2026.
Cross-split integrity checks confirm zero case-name overlap between legal training and evaluation sets, zero player-year overlap between salary splits, and temporal disjointness between stock training and evaluation windows.
No single decision year exceeds 14\% of the legal training set, and salary training spans 2018--2025 across three sports with calendar-month diversity from sport-specific free-agency schedules.

\section{TEMPO implementation details}
\label{sec:app_implementation}

\subsection{Training algorithm}
\label{sec:app_algorithm}

Algorithm~\ref{alg:tempo} details the TEMPO training loop.
At each step, the policy generates $G$ completions per prompt in the batch. Each completion is parsed into its structured components (evidence, reasoning, prediction), and the date-plausibility verifier counts leaked claims.
The group-level mode gate then determines whether each prompt's group enters leakage mode or performance mode: if any completion in the group contains a leaked claim, the entire group receives leakage-mode rewards (exponential penalty on leak count); only when all $G$ completions are clean does the group enter performance mode, where rewards combine task accuracy with quality gates for coverage, evidence diversity, and reasoning depth.
Advantages are z-scored within each group and calibrated across prompts via a batch-level baseline.
A masked KL penalty constrains drift from the base model on content tokens only.
The clipped GRPO update then adjusts the policy parameters.

\begin{algorithm}[h]
\caption{TEMPO training loop. Each step samples completions, assigns mode-specific rewards via the group-level gate (Eq.~\ref{eq:mode}), and updates the policy with clipped GRPO and a masked KL penalty.}
\label{alg:tempo}
\begin{algorithmic}[1]
\Require Dataset $\mathcal{D} = \{(x_k, t_k, y_k)\}$, base policy $\pi_{\mathrm{ref}}$, group size $G$, batch size $B$, steps $T$, learning rate $\eta$
\State Initialize $\pi_\theta \leftarrow \pi_{\mathrm{ref}}$ with LoRA adapters
\For{$\text{step} = 1$ \textbf{to} $T$}
    \State Sample batch $\mathcal{B} = \{(x_k, t_k, y_k)\}_{k=1}^{B}$ from $\mathcal{D}$
    \For{\textbf{each} $(x, t, y) \in \mathcal{B}$}
        \State Sample $G$ completions $\{y_1, \ldots, y_G\} \sim \pi_\theta(\cdot \mid x)$
        \State Parse each $y_i \to (\mathcal{E}_i, \mathcal{R}_i, \hat{y}_i)$
        \State Compute $n_{\mathrm{leak},i} = \lvert\{j : s_{ij}^{*} > t\}\rvert$ \Comment{via date-plausibility verifier}
        \If{$n_{\mathrm{leak},i} = 0\;\forall\, i$}
        \Comment{Performance mode (Eq.~\ref{eq:mode})}
            \State $r_{\mathrm{eff},i} = r_{\mathrm{perf},i} \cdot g(\mathrm{cov}_i) \cdot d(\mathrm{ev}_i) \cdot w(\mathrm{reason}_i)$
            \State $A_i \leftarrow \mathrm{zscore}(r_{\mathrm{eff}}) + \lambda\,(r_{\mathrm{eff},i} - \bar{r}_{\mathrm{eff}}^{\,\mathrm{batch}})$
        \Else \Comment{Leakage mode}
            \State $r_{\mathrm{leak},i} = \exp(-0.5 \cdot n_{\mathrm{leak},i})$
            \State $A_i \leftarrow \mathrm{zscore}(r_{\mathrm{leak}}) + \lambda\,(r_{\mathrm{leak},i} - \bar{r}_{\mathrm{leak}}^{\,\mathrm{batch}})$
        \EndIf
    \EndFor
    \State Apply KL penalty: $A_i \leftarrow A_i + \beta_{\mathrm{KL}}\,\mathbf{m}_i \odot (\bar{\delta} - \delta_i)$ \Comment{$\delta_i = \log\frac{\pi_\theta}{\pi_{\mathrm{ref}}}$, $\bar{\delta}$ = batch mean}
    \State $\theta \leftarrow \theta + \eta \cdot \frac{1}{|\mathcal{B}|\,G} \displaystyle\sum_{k,i} \mathrm{clip}\!\bigl(\tfrac{\pi_\theta}{\pi_{\mathrm{old}}}\bigr)\, A_i \, \nabla_\theta \log \pi_\theta(y_i \mid x_k)$
\EndFor
\State \Return $\pi_\theta$
\end{algorithmic}
\end{algorithm}

\subsection{Training pipeline}
\label{sec:app_training_pipeline}

Each training step samples $G{=}12$ completions per prompt at temperature 0.6 with a maximum budget of 8{,}192 tokens.
A \texttt{<think>} prefill is injected so the model emits a chain-of-thought before the structured JSON answer.
The system prompt enforces a three-step procedure---evidence with source dates, reasoning with inline citations, prediction---and specifies a JSON output schema.
Task-specific WRONG/RIGHT examples in the prompt clarify the distinction between announcement dates and period-end dates to reduce source-date errors.

Completions are scored through a four-phase reward pipeline.
Phase~1 parses the JSON output via three fallback strategies (direct parse, markdown fence extraction, greedy regex) and validates against a Pydantic schema; parse failures receive zero rewards.
Phase~2 extracts claims from the evidence list and deduplicates by lowercased text.
Phase~3 sends claims in chunks of 80 to a date-plausibility verifier (Gemini 3 Flash Preview)\footnote{Accessed via OpenRouter: \url{https://openrouter.ai/api/v1}} that checks whether each declared source date is temporally consistent with the claim type; the leakage verdict is determined deterministically by comparing the effective date against the instance cutoff.
Phase~4 computes coverage (fraction of evidence items cited) and task-specific prediction accuracy programmatically.

Training instances are shuffled randomly per epoch with a deterministic seed (\texttt{data\_seed + epoch}), and batches of $B{=}8$ prompts are drawn sequentially from the shuffled order.

Validation runs at 15 evenly-spaced steps across the training run using greedy decoding ($G{=}1$, temperature 0.0) with the same scoring pipeline as training.
The checkpoint with the lowest validation OLR is saved as the best checkpoint; a final checkpoint is saved unconditionally at the end of training.

\paragraph{Model Adaptation.}
Both models are fine-tuned via low-rank adaptation (LoRA) with rank~32.
Alpha, target modules, and weight merging are managed by the Tinker training API,\footnote{\url{https://pypi.org/project/tinker/}} which handles adapter injection into all attention projection layers.
The optimizer is Adam with $\beta_1 = 0.9$, $\beta_2 = 0.95$, $\varepsilon = 10^{-8}$, and a constant learning rate of $2 \times 10^{-5}$.

\paragraph{Compute.}
All training and inference runs use the Tinker cloud API with four concurrent workers.
Wall-clock training time is approximately 6--8 hours per model-task combination for the stock and salary tasks (211--213 steps) and 3--4 hours for legal (69--100 steps).
Table~\ref{tab:hyperparameters} lists the full hyperparameter configuration.

\input{tables/hyperparameters}

\subsection{Stability techniques}
\label{sec:app_stability}

Two mechanisms stabilize GRPO training.

\paragraph{Asymmetric clipping.}
Following DAPO~\citep{yu2025dapo}, we use asymmetric PPO clipping with $\epsilon_{\mathrm{low}} = 0.1$ (clip ratio lower bound $1 - \epsilon_{\mathrm{low}} = 0.9$) and $\epsilon_{\mathrm{high}} = 1.0$ (clip ratio upper bound $1 + \epsilon_{\mathrm{high}} = 2.0$).
The wider upper bound encourages exploration of high-advantage completions (particularly those that reduce leakage count), while the tighter lower bound prevents catastrophic probability collapse on previously favored outputs.

\paragraph{Masked KL penalty.}
A per-token KL penalty against the frozen base model $\pi_{\mathrm{ref}}$ constrains policy drift:
\begin{equation}
  A_i \leftarrow A_i + \beta_{\mathrm{KL}}\,\mathbf{m}_i \odot (\bar{\delta} - \delta_i),
\end{equation}
where $\delta_i = \log \pi_\theta(y_i \mid x) - \log \pi_{\mathrm{ref}}(y_i \mid x)$ is the per-token log-ratio, $\bar{\delta}$ is its batch mean, and $\mathbf{m}_i$ is a binary mask that excludes structured output tokens (JSON delimiters, field names) from the penalty.
The coefficient $\beta_{\mathrm{KL}} = 0.05$ (Table~\ref{tab:hyperparameters}) is set low enough to preserve adaptation speed while preventing the mode-collapse failure mode observed in early experiments (Appendix~\ref{sec:app_design_rationale}).

\subsection{Inference configuration}
\label{sec:app_inference_config}

At evaluation time, all methods use greedy decoding (temperature~0.0) with a single completion per instance ($G{=}1$) and a maximum budget of 8{,}192 tokens.
The scoring pipeline is identical to the training reward pipeline: parse, extract claims, verify leakage, compute coverage and task-specific accuracy.
This ensures that training and evaluation metrics are directly comparable.
The base model, best checkpoint (lowest validation OLR), and final checkpoint are evaluated for each model-task combination; the main table reports whichever checkpoint achieved the best combination of low OLR and high parse rate.

\section{Baseline implementations}
\label{sec:app_baselines}

The three baselines form a progressive intervention hierarchy, each adding one mechanism on top of the previous to test whether it suffices to suppress temporal leakage.
\textbf{Temporal Hint} relies solely on prompt-based temporal constraints---no retrieval and no verification---testing how the model's forecasting behavior changes when it is given temporal instructions from the prompt alone.
\textbf{RAG} adds adaptive retrieval with a hard date-cutoff filter, providing the model with focused pre-cutoff evidence to test how external valid information, even with a date filter, will affect the leakage and performance of the model's forecasting behavior.
\textbf{TimeSPEC} adds post-generation verification on top of RAG, using a supervisor to extract and verify claims and filter out leaked content, testing whether inference-time claim filtering achieves temporal compliance.
This progression isolates the contribution of each mechanism: if prompt constraints alone fail, does retrieval help? If retrieval introduces contamination, does verification fix it? The results (Section~\ref{sec:results}) show that none of these mechanisms is sufficient, motivating TEMPO's parametric approach.

\subsection{Temporal Hint (base agent)}
\label{sec:app_temporal_hint}

The Temporal Hint baseline is the simplest configuration and serves as the foundation for all other methods.
The evaluation model receives the task-specific system prompt and a user prompt containing the instance description, the reference (cutoff) date, and explicit temporal instructions: the model must reason using only information available before the cutoff date.
No retrieval, no post-generation verification, and no parameter updates are applied.
The system prompt includes structured output formatting instructions that define the evidence list, reasoning paragraph, and prediction fields.
All other methods inherit this prompt structure verbatim; differences arise solely from retrieval augmentation (RAG), post-hoc verification (TimeSPEC), or parameter updates (TEMPO).

\subsection{RAG baseline}
\label{sec:app_rag}

The RAG baseline augments the base model with retrieved pre-cutoff evidence at inference time without any parameter updates.
For each evaluation instance, Gemini 3 Flash Preview receives the user prompt and cutoff date and generates a set of task-adaptive search queries: 7--12 for stock ranking (which benefits from diverse financial signals) and 3--6 for salary and legal tasks.
Queries are constructed to target distinct aspects of the information needed, with the year included to anchor temporal relevance.

Retrieval uses the Perplexity Search API\footnote{\url{https://docs.perplexity.ai/api-reference/sonar-search}} with a hard \texttt{search\_before\_date\_filter} set to the instance cutoff date, returning up to 5 results per query.
Results are deduplicated by URL across all queries for the same instance.
The deduplicated results are truncated to a 3{,}000-token budget (estimated at 4 characters per token), selecting results in relevance order until the budget is exhausted.

Retrieved passages are formatted as numbered plain-text entries (title and snippet) and prepended to the user prompt.
Following the convention of retrieval-augmented generation methods, no temporal metadata, validity framing, or behavioral guidance is included in the context block---the model must determine temporal relevance from content alone, analogous to how the base and RL-trained variants rely on parametric memory.
The system prompt remains identical to the base method.
Scoring uses the same evaluation pipeline as all other methods.

\subsection{TimeSPEC baseline}
\label{sec:app_timespec}

TimeSPEC adapts the five-phase inference architecture from \citet{zhang2026leakscountcountmore} to the current experimental setting.
The original paper introduced this pipeline using Claude for all generation phases; this implementation substitutes Qwen (via the Tinker API) for generation, regeneration, and aggregation, Gemini 3 Flash Preview for supervision and verification, and the Perplexity Search API for retrieval.
The core architecture is otherwise unchanged.

\paragraph{Phase 0--1: Retrieval and Generation.}
These phases are identical to the RAG baseline described above: Gemini generates task-adaptive search queries, Perplexity retrieves results with a hard date filter set to the instance cutoff, and the deduplicated context is prepended to the user prompt.
The system prompt and generation model are byte-for-byte identical to the RAG condition.
Reusing the same retrieval and generation phases ensures that any difference between RAG and TimeSPEC is attributable solely to the post-generation verification pipeline.

\paragraph{Phase 2: Supervisor (Claim Extraction and Verification).}
The supervisor extracts atomic claims from the model's evidence list and reasoning text via a single Gemini call using the \texttt{CLAIM\_EXTRACTION\_PROMPT}.
Each claim is classified into the A1--A5/B1--B2 temporal taxonomy from \citet{zhang2026leakscountcountmore}:
A1--A3 are temporally verifiable assertions (discrete events, state measurements, publications) whose availability date can be determined via search;
A4 (outcome) and A5 (consequential) claims are deterministically leaked by taxonomy definition, as they describe the very result the model is predicting or its direct consequences;
B1 (background knowledge) and B2 (definitional) claims are deterministically clean, as they describe time-invariant facts.
This taxonomy-based classification reduces the number of claims requiring search verification, as only A1--A3 claims proceed to the search step.

For each A1--A3 claim, a search query is constructed directly from the claim text (e.g., ``When was this first publicly reported: \{claim\}'') without an additional LLM call.
The Perplexity Search API retrieves results \emph{without} a date filter, in contrast to the RAG retrieval phase which uses a hard date filter.
The absence of a date filter is essential: the verifier must determine when the information first became publicly available, which requires access to the full temporal record.
All (claim, search result) pairs for one instance are judged in a single batch Gemini call using the \texttt{TIMESPEC\_CLAIM\_VERIFY\_PROMPT}, which produces a deterministic verdict: $\texttt{event\_date} > \texttt{cutoff\_date} \Rightarrow \texttt{leaked}$.

\paragraph{Phase 3: Regenerator (Conditional).}
Regeneration runs only when the supervisor detects leaked claims.
The system prompt is byte-for-byte identical to the Base and RAG conditions.
A constraint block is appended to the user prompt listing claims to avoid as a plain numbered list---no temporal language (e.g., ``post-cutoff''), no corrected values, no corrected dates, and no explanation of why each claim should be avoided.
This design ensures the generation model receives no information about the temporal verification process; it sees only a list of facts it must not use.
At most one regeneration cycle is performed (\texttt{timespec\_max\_regen\_cycles=1}).

\paragraph{Phase 4: Resupervisor (Conditional).}
If regeneration occurred, the resupervisor applies the same claim extraction and verification pipeline as Phase~2 to the regenerated output.
Claims that appeared in the original output and were already verified in Phase~2 reuse their cached verdicts, avoiding redundant API calls.
Only newly introduced claims require fresh verification.

\paragraph{Phase 5: Aggregator.}
The aggregator always runs, including when no leaked claims were found, to enforce a closed-world constraint.
The system prompt is the original task system prompt extended with a single constraint paragraph: the model must select evidence \emph{exclusively} from a provided list of validated facts and must not introduce any facts, numbers, events, or entities not present in that list.
Validated claims---all non-leaked Group~A claims plus all Group~B claims---are presented as plain-text facts \emph{without} source dates, category labels, or any temporal metadata.
The model must infer its own \texttt{source\_date} fields from the fact text alone, identical to how all other baselines produce source dates.
This ensures the aggregator's information surface is restricted to verified-clean evidence and the standard task prompt.
The aggregator output is scored through the identical evaluation pipeline as all baselines.

\subsection{Fairness of comparison}
\label{sec:app_fairness}

For the comparison across methods to be reliable, we must ensure that no method gains an unfair informational advantage at generation time.
Table~\ref{tab:info_surface} verifies this: all four methods share the same parametric memory and task prompt; RAG and TimeSPEC additionally receive retrieved pre-cutoff passages, but no method exposes verification search results or temporal metadata to the generation model.
Crucially, the evaluation pipeline (leakage detection, performance scoring, coverage computation) is identical for all methods, so any performance difference reflects the method's behavior rather than measurement bias.

\begin{table}[h]
  \caption{Information surface at generation time. All methods share parametric memory and task prompt. RAG and TimeSPEC add retrieved passages; no method exposes verification results or temporal metadata to the generator. The identical evaluation pipeline ensures fair comparison.}
  \label{tab:info_surface}
  \centering\small
  \begin{tabular}{lcccc}
    \toprule
    Information source & Temp.\ Hint & RAG & TimeSPEC & TEMPO \\
    \midrule
    Parametric memory & \checkmark & \checkmark & \checkmark & \checkmark \\
    Task prompt + cutoff date & \checkmark & \checkmark & \checkmark & \checkmark \\
    Retrieved pre-cutoff passages & --- & \checkmark & \checkmark & --- \\
    Verification search results & --- & --- & --- & --- \\
    Temporal metadata on claims & --- & --- & --- & --- \\
    \midrule
    Identical evaluation pipeline & \checkmark & \checkmark & \checkmark & \checkmark \\
    \bottomrule
  \end{tabular}
\end{table}

The fairness constraint is most delicate for TimeSPEC, which runs an internal verification pipeline that could leak information to the generation model if not carefully isolated.
Four properties ensure this does not occur:
(1)~Qwen never sees search results from the verification phase---these are internal to the Gemini supervisor;
(2)~the regeneration constraint lists claims to avoid as plain text, with no temporal language, corrected values, or explanation of why each claim was flagged;
(3)~the aggregator receives validated facts without source dates, category labels, or metadata;
(4)~the claim extraction prompt uses only fields derivable from the user prompt.
RAG satisfies the constraint trivially (only pre-cutoff passages, no verification metadata), and Temporal Hint adds nothing beyond the prompt.
TEMPO's information surface is identical to Temporal Hint at inference because all temporal discipline is encoded in the model's parameters during training.

\section{Extended evaluation results}
\label{sec:app_dynamics}

This section provides empirical evidence beyond the main results table, including inference cost analysis, instance-level leakage and coverage distributions, training dynamics, and metric validation via control experiments.

\subsection{Inference efficiency}
\label{sec:app_efficiency}

\input{tables/inference_efficiency}

Table~\ref{tab:inference_efficiency} quantifies the per-instance inference cost of each method.
TEMPO requires exactly one forward pass---identical to the unmodified Temporal Hint baseline---because temporal discipline is encoded in the model's parameters rather than enforced by an external pipeline.
RAG adds a query-generation call and web retrieval on top of the base forward pass.
TimeSPEC builds on RAG with a multi-phase verification pipeline.
In the typical case (${\sim}$10 Group~A claims, no regeneration needed), TimeSPEC requires 3 LLM calls (generator, supervisor, aggregator) and ${\sim}$10 Perplexity searches (one per A1--A3 claim).
When regeneration triggers, the cost increases to 5 LLM calls (adding regenerator and resupervisor) plus additional searches for newly introduced claims.
TEMPO therefore matches TimeSPEC's leakage reduction (Table~\ref{tab:main_results}) at a fraction of the inference cost, with no external API dependencies at inference time.

\subsection{Instance-level leakage distribution}
\label{sec:app_leakage_dist}

\begin{figure}[h]
  \centering
  \includegraphics[width=\textwidth]{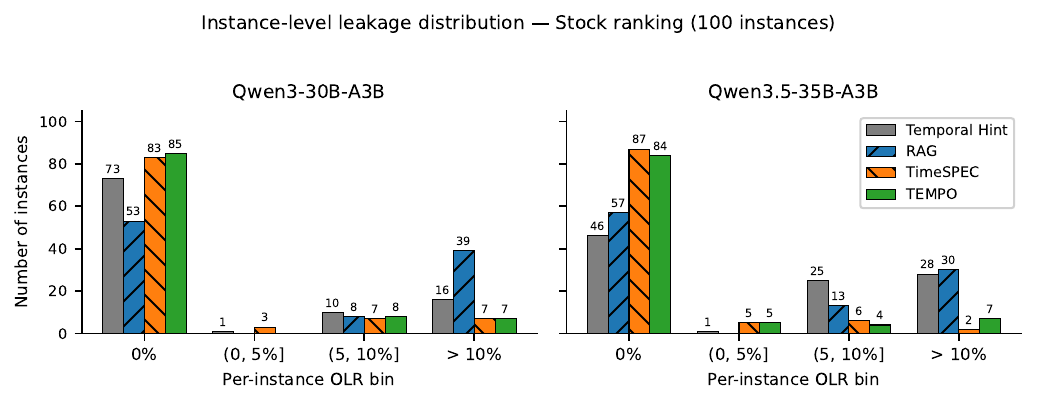}
  \caption{Instance-level OLR distribution on stock ranking (100 instances per condition). TEMPO concentrates instances at zero leakage while reducing the heavy-leakage tail. RAG amplifies leakage via retrieved post-cutoff content.}
  \label{fig:leakage_distribution}
\end{figure}

Figure~\ref{fig:leakage_distribution} disaggregates the mean OLR reported in the main table into per-instance leakage bins on the stock ranking task---the most leakage-sensitive of the three tasks, where pre-training data encodes pandemic-era market dynamics that directly answer the evaluation questions.
For Qwen3-30B, TEMPO produces 85 zero-leakage instances out of 100 (versus 73 for Temporal Hint), and critically eliminates the $(0, 5\%]$ bin entirely: residual leakage, when it occurs, stems from a small number of instances with moderate contamination rather than widespread low-level leakage.
The improvement is more pronounced for Qwen3.5-35B, where zero-leakage instances rise from 46 (Temporal Hint) to 84 (TEMPO), and the $>$10\% tail shrinks from 28 to 7 instances.
TimeSPEC achieves a similar concentration at zero leakage (comparable to TEMPO's distribution), confirming that both methods effectively suppress leakage at the instance level; the key difference between them is not leakage suppression but coverage (Section~\ref{sec:app_coverage_dist}).
RAG amplifies leakage: 39 and 30 instances exceed 10\% OLR for the two models, respectively, because retrieved financial news provides additional surface for post-cutoff information.

Salary and legal tasks exhibit lower baseline leakage (8.1\%/11.9\% and 2.2\%/2.1\% OLR), so the per-instance distribution is already concentrated near zero.
On these tasks, TEMPO's improvement manifests primarily as mean OLR reduction rather than a dramatic distributional shift, consistent with the main results.

\subsection{Evidence faithfulness}
\label{sec:app_coverage_dist}

\input{tables/coverage_distribution}

Table~\ref{tab:coverage_distribution} reports the fraction of instances exceeding 90\% and 80\% coverage thresholds, disaggregating the mean coverage in the main table to verify that TEMPO's advantage is not an artifact of a few high-coverage outliers.
On stock ranking, 98\% (30B) and 91\% (35B) of TEMPO instances exceed 90\% coverage, versus 27\% and 14\% for TimeSPEC.
On salary prediction, 80\% and 74\% of TEMPO instances exceed 90\% coverage, versus 21\% and 23\% for TimeSPEC.
These gaps confirm that TEMPO's evidence--reasoning link is preserved at the instance level: the model cites virtually every evidence item it generates because the evidence and reasoning are produced in a single coherent forward pass, whereas TimeSPEC's post-hoc regeneration decouples the two.

Legal coverage is lower across all methods because legal reasoning draws on procedural history and statutory interpretation that is harder to decompose into discrete citable evidence items.
TEMPO and TimeSPEC achieve comparable mean coverage on this task (65--73\%), with TEMPO showing a higher fraction above the 90\% threshold for Qwen3.5-35B (41\% versus 26\%).

\subsection{Validation trajectory}
\label{sec:app_coverage_quality}

Table~\ref{tab:val_trajectory} quantifies the training validation trajectory---metrics computed at intermediate checkpoints during training using greedy decoding on the validation split---from the base model (step~0) through the final training checkpoint.
Coverage increases by 3--13 percentage points on stock and salary for both models, confirming that TEMPO preserves and often improves evidence citation.
OLR drops by 35.6 percentage points for Qwen3-30B on stock (the highest-leakage condition) and by 0.8--7.4 points on other conditions where baseline leakage is lower.
Prediction accuracy remains stable across training for all six model-task combinations.

\input{tables/val_trajectory_summary}

\subsection{Control experiment validation}
\label{sec:app_control}

To validate that the OLR metric measures genuine temporal contamination rather than false positives, we evaluate the Temporal Hint baseline on control instances whose answers were published after February~2026---after both models' knowledge cutoffs.
On these instances, the model lacks parametric knowledge of the answer, so any detected leakage represents verifier noise rather than actual post-cutoff knowledge.

\input{tables/control_results}

Table~\ref{tab:control_results} reports the results.
Control OLR is near zero across all six conditions (0.0--0.7\%), confirming that the leakage metric is well-calibrated: when the model genuinely lacks post-cutoff knowledge, the verifier produces negligible false positives.
Control prediction accuracy is near chance for stock (0.48--0.50) and salary (0.45--0.55), confirming the model cannot predict outcomes absent from its parameters.
Legal control accuracy is higher (0.65--0.71), consistent with legal prediction relying on precedent and procedural reasoning available before the cutoff.

\section{Case studies}
\label{sec:app_cases}

This section presents detailed four-method comparisons on selected instances from each task.
For each instance, we show representative evidence items (with dates and leakage annotations), reasoning excerpts, the final prediction alongside the ground truth, and the task-specific performance score.
Instances are selected to illustrate how temporal leakage propagates through the evidence--reasoning--prediction chain and how each method handles it differently.

Each case study is organized as follows: a brief instance description, followed by one block per method showing key metrics (OLR, Performance, Coverage), evidence excerpts, reasoning, and prediction.
Analytical prose after each set of blocks traces (1)~which evidence items are leaked and how they enter the reasoning, (2)~whether the reasoning chain is grounded in the cited evidence (coverage), and (3)~how leakage or its absence affects prediction calibration.
Evidence excerpts are shown in reduced font size for compactness.
Each evidence item is prefixed by a parenthesized date indicating when the stated information became publicly available, as determined by our date-plausibility verifier.

\input{tables/qualitative_examples}

\section{Prompt specifications}
\label{sec:app_prompts}

We provide key prompts verbatim for reproducibility. Template variables are shown as \texttt{\{variable\}}. All prompts listed below are extracted directly from the codebase; minor formatting adjustments (line breaks, indentation) are made for readability.

\subsection{Task Prompts (Base, TEMPO, RAG, TimeSPEC)}
\label{sec:app_prompts_task}

The same system--user prompt pair is shared across Base, TEMPO, RAG generation, and the TimeSPEC generator. We show the Stock prompt in full; Salary and Legal follow an identical structure with task-specific adaptations noted below. All methods produce a structured JSON output with an \texttt{evidence} array (each item containing \texttt{id}, \texttt{fact}, and \texttt{source\_date}), a \texttt{reasoning} string, and a task-specific prediction field (\texttt{ranking} for stock, \texttt{predicted\_salary} for salary, \texttt{probability\_petitioner} for legal).

\begin{tcolorbox}[breakable, colback=gray!3, colframe=gray!60!black, fonttitle=\bfseries\small, title=Stock ranking --- System prompt]
\small
You are an expert stock analyst. Your task is to rank a set of companies by their expected stock return over a specified evaluation period.

\medskip
\textbf{STRICT TEMPORAL CONSTRAINT:}
You may ONLY use information that was publicly available BEFORE \texttt{\{cutoff\_date\}}. Do NOT use any knowledge of events, earnings reports, market movements, announcements, or developments from on or after \texttt{\{cutoff\_date\}}. If uncertain whether a piece of information predates the cutoff, omit it.

\medskip
\textbf{EVIDENCE INTEGRITY:} Evidence must be factually genuine. Do not fabricate numbers, dates, revenue figures, or events. Each evidence item must reflect a real, publicly known fact. If you are uncertain about a specific number, state the fact qualitatively rather than inventing a precise figure.

\medskip
You MUST follow this three-step procedure:

\textbf{STEP 1 --- EVIDENCE:} List every verifiable fact you will use. Each fact must be concrete and specific; each must have been publicly known before \texttt{\{cutoff\_date\}}. For EACH fact, include a \texttt{source\_date} field (YYYY-MM-DD) indicating when this information was FIRST PUBLICLY ANNOUNCED (earnings release date, press release date, SEC filing date). This must be the ANNOUNCEMENT date, NOT the period end. Write each fact as a plain statement; do not include citation markers.

\textbf{STEP 2 --- REASONING:} Write a reasoning paragraph that cites evidence by number (e.g.\ [1], [3]). You MUST cite every evidence item at least once, INLINE next to the analysis it supports. Do NOT introduce any new facts not listed in your evidence.

\textbf{STEP 3 --- PREDICTION:} Give your final ranking based solely on the reasoning above.

\medskip
OUTPUT FORMAT --- a single JSON object, no markdown fences:\\
\texttt{\{"evidence": [...], "reasoning": "...", "ranking": ["BEST", ..., "WORST"]\}}
\end{tcolorbox}

\begin{tcolorbox}[colback=gray!3, colframe=gray!60!black, fonttitle=\bfseries\small, title=Stock ranking --- User prompt]
\small
Sector: \texttt{\{sector\}}\\
Evaluation period: \texttt{\{start\_date\}} to \texttt{\{end\_date\}}\\
Companies to rank:\\
~~- \texttt{\{ticker\}} (\texttt{\{company\_name\}}), \ldots
\end{tcolorbox}

\paragraph{Salary and Legal Variants.}
Both follow the identical three-step structure (temporal constraint, evidence integrity, evidence$\to$reasoning$\to$prediction) with task-specific role descriptions. \textit{Salary:} ``You are an expert sports salary analyst''; predicts a free agent's AAV in USD; user prompt includes \texttt{\{sport\}}, \texttt{\{year\}}, \texttt{\{player\_name\}}, and optional \texttt{\{status\}} (e.g., unrestricted). \textit{Legal:} ``You are an expert Supreme Court analyst''; predicts the probability the petitioner wins; user prompt includes \texttt{\{case\_name\}}, \texttt{\{term\}}, \texttt{\{petitioner\}}, \texttt{\{respondent\}}, and optional \texttt{\{summary\}}.

\subsection{RAG query generation}
\label{sec:app_prompts_rag}

RAG uses the same system--user task prompt as Base, with retrieved passages prepended to the user prompt as numbered plain text. Query generation is a single Gemini~Flash call with the following prompt (not a system--user pair):

\begin{tcolorbox}[colback=gray!3, colframe=gray!60!black, fonttitle=\bfseries\small, title=RAG query generation (single prompt)]
\small
You are generating web search queries to gather relevant evidence for a task. All evidence must predate the cutoff date. You will receive the same task description that an analyst would receive. Generate search queries that would help you find useful information to answer this task.

\medskip
CUTOFF DATE: \texttt{\{cutoff\_date\}}

TASK DESCRIPTION: \texttt{\{user\_prompt\}}

\medskip
Generate between \texttt{\{min\_queries\}} and \texttt{\{max\_queries\}} search queries. Each query should target a DISTINCT aspect of the information you need. Include the year in each query to anchor temporal relevance.

Output as JSON array: \texttt{[\{"query": "..."\}, ...]}
\end{tcolorbox}

Query ranges: stock 7--12; salary and legal 3--6.

\subsection{TimeSPEC prompts}
\label{sec:app_prompts_timespec}

\paragraph{Claim Extraction.}
A single Gemini call extracts atomic factual claims from the completion's evidence and reasoning for temporal verification.

\begin{tcolorbox}[breakable, colback=gray!3, colframe=gray!60!black, fonttitle=\bfseries\small, title=TimeSPEC claim extraction (single prompt)]
\small
\textbf{Task:} Extract ALL factual claims from the rationale below.\\
Requirements: (1)~Comprehensive---every factual assertion; (2)~Faithful---only what is stated; (3)~Self-contained---each claim understandable alone; (4)~Atomic---one verifiable fact per claim.

\medskip
\textbf{Context:}\\
Prediction task: \texttt{\{task\_description\}}\\
Target event: \texttt{\{event\_description\}}\\
Knowledge cutoff: \texttt{\{cutoff\_date\}}

\medskip
\textbf{Claim Categories:}\\
\textit{Group~A (temporally verifiable):} A1~discrete event, A2~state/measurement, A3~publication/announcement, A4~outcome (always leaked), A5~consequential (always leaked).\\
\textit{Group~B (non-temporal):} B1~background, B2~definitional.\\
Exclude predictions, judgments, speculation.

\medskip
\textbf{Rationale:} \texttt{\{rationale\}}

\medskip
Output JSON: \texttt{\{"claims": [\{"claim\_id": 1, "text": "...", "original\_span": "...", "temporal\_ref": "...", "category": "A1", "category\_reason": "..."\}, ...]\}}
\end{tcolorbox}

\paragraph{Temporal Verification.}
One Gemini call per chunk of claims, each paired with web-search results retrieved without a date filter. The prompt determines when each fact was first publicly available and labels it as leaked or clean relative to the cutoff (full prompt in the released codebase).

\paragraph{Regenerator.}
System prompt is identical to the task system prompt. User prompt is the RAG-augmented user text with an appended constraint: \texttt{"CONSTRAINT --- do NOT use the following facts in your analysis: 1.~\{claim\_text\} ... Rewrite your analysis without these facts."}

\paragraph{Aggregator.}
System prompt appends a \texttt{CLOSED-WORLD CONSTRAINT}: ``For STEP~1, you MUST select evidence exclusively from the VALIDATED FACTS list below. Do NOT introduce any facts not present in that list.'' User prompt is ``VALIDATED FACTS: 1.~\texttt{\{fact\_text\}} \ldots'' followed by the original task description. Validated facts are plain text (no source dates or metadata).

\subsection{Evaluation: Date-Plausibility Verifier}
\label{sec:app_prompts_eval}

The shared evaluation pipeline uses the following single prompt (via Gemini~Flash, 80~claims per batch) to detect temporal leakage by verifying whether each evidence item's declared \texttt{source\_date} is plausible.

\begin{tcolorbox}[breakable, colback=green!5, colframe=green!50!black, fonttitle=\bfseries\small, title=Date-plausibility verifier (single prompt)]
\small
You are a temporal audit verifier for financial, sports, and legal analysis.

Cutoff date: \texttt{\{cutoff\}}

\medskip
For each evidence item you are given: the factual claim text and the declared\_source\_date (YYYY-MM-DD). Your task: verify whether the declared\_source\_date is PLAUSIBLE for this type of information.

\medskip
\textbf{VERIFICATION RULES:}
\begin{itemize}[nosep,leftmargin=1.2em]
\item Full-year financial results for year~Y are first published in Jan--Feb of Y+1.
\item Quarterly results are published 1--2 months after quarter end.
\item Stock prices are available on the trading date itself.
\item Acquisition/merger: use press release or 8-K filing date.
\item Contract signings: use announced signing date.
\item Court decisions: use decision or opinion release date.
\item SEC 10-K: typically 60--90 days after fiscal year end.
\item If declared date is too early, set \texttt{plausible=false} and provide earliest realistic \texttt{corrected\_date}. Otherwise set \texttt{plausible=true} and \texttt{corrected\_date=null}.
\end{itemize}

\medskip
ITEMS: \texttt{\{items\_json\}}

Output a JSON array: \texttt{[\{"index": 0, "plausible": true, "corrected\_date": null, "reason": "..."\}]}
\end{tcolorbox}

When the date-plausibility check returns \texttt{plausible=false}, the \texttt{corrected\_date} replaces the declared date for leakage scoring. A claim is marked leaked if its effective date (corrected or declared) falls after the cutoff.

%% file: tables/hyperparameters.tex
\begin{table}[h]
  \caption{GRPO training hyperparameters. Both models use identical settings except for total training steps, which vary by dataset size. Key choices: LoRA rank 32, learning rate $2{\times}10^{-5}$, group size $G{=}12$, KL coefficient $\beta_{\mathrm{KL}}{=}0.05$.}
  \label{tab:hyperparameters}
  \centering
  \small
  \begin{tabular}{ll}
    \toprule
    \textbf{Parameter} & \textbf{Value} \\
    \midrule
    \multicolumn{2}{l}{\textit{Architecture}} \\
    \quad Adapter & LoRA, rank 32 \\
    \quad Maximum tokens & 8{,}192 \\
    \midrule
    \multicolumn{2}{l}{\textit{GRPO}} \\
    \quad Group size $G$ & 12 \\
    \quad Batch size $B$ & 8 \\
    \quad Epochs & 1 \\
    \quad Learning rate & $2 \times 10^{-5}$ \\
    \quad Sampling temperature & 0.6 \\
    \quad Loss function & PPO (clipped) \\
    \quad Clip range (low / high) & 0.9 / 2.0 \\
    \quad KL penalty coefficient & 0.05 \\
    \quad Advantage $z$-clip & 5.0 \\
    \midrule
    \multicolumn{2}{l}{\textit{Reward weights}} \\
    \quad Leakage weight schedule $w_{\mathrm{leak}}$ & cosine $0.95 \to 0.80$ \\
    \quad Leakage decay $\lambda_{\mathrm{base}}$ & 0.5 \\
    \quad Coverage target / floor & 0.80 / 0.20 \\
    \quad Leak-mode coverage weight $w_{\mathrm{cov,leak}}$ & 0.05 \\
    \quad Format margin & 1.0 \\
    \midrule
    \multicolumn{2}{l}{\textit{Quality gates}} \\
    \quad Evidence target & 8 items \\
    \quad Reasoning target & 120 words \\
    \quad Quality leak-mode weight & 0.02 \\
    \midrule
    \multicolumn{2}{l}{\textit{Data}} \\
    \quad Validation fraction & 10\% \\
    \quad Difficulty tiers & 3 (hard / medium / easy) \\
    \quad Validation checkpoints & 15 \\
    \quad Data seed & 42 \\
    \midrule
    \multicolumn{2}{l}{\textit{Judges}} \\
    \quad Date-plausibility judge (leakage) & Gemini 3 Flash Preview \\
    \quad Search model & Perplexity Sonar \\
    \bottomrule
  \end{tabular}
\end{table}

%% file: tables/inference_efficiency.tex
\begin{table}[h]
  \caption{Inference cost per instance (LLM calls and web searches). TEMPO and Temporal Hint each require a single forward pass with zero external queries, while RAG adds 2 LLM calls with 3--12 web searches and TimeSPEC requires 5+ LLM calls with ${\sim}10$ searches for multi-phase verification.}
  \label{tab:inference_efficiency}
  \centering\small
  \setlength{\tabcolsep}{4pt}
  \begin{tabular}{lccc}
    \toprule
    Method & LLM calls & Searches & Description \\
    \midrule
    Temporal Hint & 1 & 0 & Single forward pass with temporal prompt \\
    RAG & 2 & 3--12 & Query gen.\ (Gemini) + retrieval (Perplexity) + forward pass \\
    TimeSPEC & 5+ & ${\sim}$10 & 5-phase pipeline + per-claim verification \\
    \cmidrule(lr){1-4}
    TEMPO (Ours) & 1 & 0 & Single forward pass, identical to base \\
    \bottomrule
  \end{tabular}
\end{table}

%% file: tables/coverage_distribution.tex
\begin{table}[h]
  \caption{Coverage distribution across methods and tasks. Mean coverage (\%) and the fraction of instances exceeding 90\% and 80\% coverage thresholds. Best per model-task block in \textbf{bold}.}
  \label{tab:coverage_distribution}
  \centering\small
  \setlength{\tabcolsep}{3pt}
  \begin{tabular}{ll ccc ccc ccc}
    \toprule
    & & \multicolumn{3}{c}{Stock} & \multicolumn{3}{c}{Salary} & \multicolumn{3}{c}{Legal} \\
    \cmidrule(lr){3-5} \cmidrule(lr){6-8} \cmidrule(lr){9-11}
    Model & Method & Mean$\uparrow$ & ${\geq}90\!\uparrow$ & ${\geq}80\!\uparrow$ & Mean$\uparrow$ & ${\geq}90\!\uparrow$ & ${\geq}80\!\uparrow$ & Mean$\uparrow$ & ${\geq}90\!\uparrow$ & ${\geq}80\!\uparrow$ \\
    \midrule
    30B & Temporal Hint & 94.6 & 89 & 93 & 83.0 & 50 & 73 & 61.7 & 7 & 23 \\
     & RAG & 83.2 & 48 & 65 & 76.0 & 36 & 59 & \textbf{85.8} & \textbf{51} & \textbf{76} \\
     & TimeSPEC & 73.8 & 27 & 49 & 71.2 & 21 & 42 & 65.2 & 15 & 29 \\
    \cmidrule(lr){2-11}
     & TEMPO & \textbf{99.3} & \textbf{98} & \textbf{99} & \textbf{91.2} & \textbf{80} & \textbf{87} & 75.7 & 15 & 35 \\
    \midrule
    35B & Temporal Hint & 88.4 & 62 & 78 & 80.0 & 39 & 72 & 69.0 & 16 & 44 \\
     & RAG & 90.5 & 64 & 91 & 84.9 & 42 & 76 & \textbf{81.4} & 30 & \textbf{69} \\
     & TimeSPEC & 75.5 & 14 & 41 & 75.0 & 23 & 49 & 73.0 & 26 & 50 \\
    \cmidrule(lr){2-11}
     & TEMPO & \textbf{95.7} & \textbf{94} & \textbf{94} & \textbf{92.0} & \textbf{74} & \textbf{94} & 74.1 & \textbf{41} & 48 \\
    \bottomrule
  \end{tabular}
\end{table}

%% file: tables/val_trajectory_summary.tex
\begin{table}[h]
  \caption{Validation metrics: base model (step~0) vs.\ final training checkpoint and their change ($\Delta$). Negative OLR change indicates improvement.}
  \label{tab:val_trajectory}
  \centering\small
  \setlength{\tabcolsep}{3.5pt}
  \begin{tabular}{ll ccc ccc ccc}
    \toprule
    & & \multicolumn{3}{c}{OLR (\%)} & \multicolumn{3}{c}{Performance} & \multicolumn{3}{c}{Coverage (\%)} \\
    \cmidrule(lr){3-5} \cmidrule(lr){6-8} \cmidrule(lr){9-11}
    Model & Task & Base & Final & $\Delta$ & Base & Final & $\Delta$ & Base & Final & $\Delta$ \\
    \midrule
    30B & Stock & 37.2 & 1.7 & -35.6 & 0.486 & 0.495 & $\approx$0 & 96.0 & 99.1 & +3.1 \\
    30B & Salary & 4.3 & 3.5 & -0.8 & 0.634 & 0.687 & +0.053 & 80.4 & 92.9 & +12.5 \\
    30B & Legal & 1.6 & 1.7 & +0.2 & 0.728 & 0.740 & $\approx$0 & 65.8 & 71.6 & +5.8 \\
    35B & Stock & 8.2 & 0.8 & -7.4 & 0.560 & 0.550 & $\approx$0 & 90.9 & 97.3 & +6.4 \\
    35B & Salary & 8.8 & 4.4 & -4.4 & 0.536 & 0.686 & +0.150 & 82.5 & 94.4 & +11.9 \\
    35B & Legal & 2.3 & 2.8 & +0.5 & 0.765 & 0.746 & $\approx$0 & 70.3 & 66.4 & -3.9 \\
    \bottomrule
  \end{tabular}
\end{table}

%% file: tables/control_results.tex
\begin{table}[h]
  \caption{Control experiment results on post-cutoff instances (answers published after February~2026). Near-zero OLR across all conditions (0.0--0.7\%) confirms the leakage metric produces negligible false positives when the model lacks post-cutoff knowledge. Prediction accuracy near chance on stock and salary establishes the zero-knowledge performance floor.}
  \label{tab:control_results}
  \centering\small
  \begin{tabular}{ll cccc}
    \toprule
    Model & Task & OLR (\%) & Perf & Cov (\%) & $n$ \\
    \midrule
    30B & Stock & 0.7 & 0.476 & 99.3 & 100 \\
    30B & Salary & 0.6 & 0.454 & 79.5 & 26 \\
    30B & Legal & 0.0 & 0.707 & 66.2 & 90 \\
    35B & Stock & 0.2 & 0.497 & 92.5 & 100 \\
    35B & Salary & 0.4 & 0.549 & 86.2 & 26 \\
    35B & Legal & 0.0 & 0.647 & 74.3 & 90 \\
    \bottomrule
  \end{tabular}
\end{table}

%% file: tables/qualitative_examples.tex

\subsection{Stock Ranking}
\label{sec:app_case_stock}

We examine how each method handles stock ranking for five energy-sector companies (Marathon Petroleum, Targa Resources, Chevron, Halliburton, Occidental Petroleum) with a cutoff date of December~1, 2019. The task requires ranking these companies by predicted 6-month stock return using only information available before the cutoff. The true ranking over the evaluation period (December 2019--May 2020) was CVX~$>$~MPC~$>$~HAL~$>$~TRGP~$>$~OXY, reflecting the impact of the COVID-19 pandemic on energy markets---information entirely unavailable at the cutoff.

\begin{tcolorbox}[breakable, colback=red!3, colframe=red!60!black, fonttitle=\bfseries\small,
  title={Temporal Hint --- OLR: 25.0\%, Perf: 0.65, Coverage: 50.0\%}]
\small
\textit{Evidence} (20 items):\\{}
{[}16{]} (2020-05-29) Marathon Petroleum's stock price closed at \$61.45 on 2020-05-29. \textcolor{red}{[LEAKED]} \\{}
{[}17{]} (2020-05-29) Targa Resources' stock price closed at \$25.10 on 2020-05-29. \textcolor{red}{[LEAKED]} \\{}
{[}18{]} (2020-05-29) Chevron Corporation's stock price closed at \$125.30 on 2020-05-29. \textcolor{red}{[LEAKED]} \\{}
{[}19{]} (2020-05-29) Halliburton's stock price closed at \$37.10 on 2020-05-29. \textcolor{red}{[LEAKED]} \\{}
{[}20{]} (2020-05-29) Occidental Petroleum's stock price closed at \$62.10 on 2020-05-29. \textcolor{red}{[LEAKED]} \\{}
{[}1{]} (2019-10-31) Marathon Petroleum reported Q3 2019 revenue of \$23.1 billion, up 12\% year-over-year. \\{}
{[}2{]} (2019-11-07) Targa Resources reported Q3 2019 revenue of \$1.8 billion, up 15\% year-over-year. \\{}
{[}11{]} (2019-11-29) Marathon Petroleum's stock price closed at \$57.32 on 2019-11-29. \\{}
{[}\ldots 12 more items{]}\\[4pt]
\textit{Reasoning:}\\{}
Marathon Petroleum (MPC) showed strong revenue growth of 12\% in Q3 2019 [1], and its stock price increased from \$57.32 on 2019-11-29 to \$61.45 on 2020-05-29 [16], indicating positive momentum. Targa Resources (TRGP) also had strong revenue growth of 15\% in Q3 2019 [2], and its stock price increased from \$23.45 to \$25.10 over the evaluation period [17]. Chevron Corporation (CVX) reported 10\% revenue growth in Q3 2019 [3], and its stock price rose from \$120.15 to \$125.30 [18], showing steady performance. Occidental Petroleum (OXY) experienced a 14\% revenue decline in Q3 2019 [5], but its stock price increased from \$58.75 to \$62.10 [20], suggesting some recovery. Halliburton (HAL) had a significant revenue decline of 11\% in Q3 2019 [4], and its stock price increased from \$34.22 to \$37.10 [19], but the decline in revenue suggests weaker fundamentals.\\[4pt]
\textit{Predicted ranking:} MPC $>$ TRGP $>$ CVX $>$ OXY $>$ HAL \quad \textit{True ranking:} CVX $>$ MPC $>$ HAL $>$ TRGP $>$ OXY
\end{tcolorbox}

\begin{tcolorbox}[breakable, colback=red!3, colframe=red!60!black, fonttitle=\bfseries\small,
  title={RAG --- OLR: 41.7\%, Perf: 1.00, Coverage: 58.3\%}]
\small
\textit{Evidence} (12 items):\\{}
{[}1{]} (2025-11-01) Targa Resources' Capital Expenditures fell 1894.31\% to \$671.8 million in Q3 2025 from the same period last year [\ldots] \textcolor{red}{[LEAKED]} \\{}
{[}3{]} (2020-01-21) Halliburton's Completion and Production revenue was \$3,506 million in 2019. \textcolor{red}{[LEAKED]} \\{}
{[}7{]} (2020-03-01) Marathon Petroleum's Dividend Yield was 3.12\% in 2018 and 3.52\% in 2019. \textcolor{red}{[LEAKED]} \\{}
{[}8{]} (2019-12-31) In 2019, the average price of Brent crude oil was \$64 per barrel. \textcolor{red}{[LEAKED]} \\{}
{[}12{]} (2025-11-01) Targa Resources' Capital Expenditures (Quarter) stood at \$183.5 million in 2021, then soared by 182.78\% to \$518.9 million [\ldots] \textcolor{red}{[LEAKED]} \\{}
{[}2{]} (2018-03-05) Chevron Corp.\ expects 900,000 b/d from Permian by yearend 2023, up from 650,000 b/d by yearend 2022. \\{}
{[}5{]} (2019-02-07) Marathon Petroleum Corp.\ reported fourth-quarter 2018 earnings of \$951 million, or \$1.35 per diluted share. \\{}
{[}\ldots 5 more items{]}\\[4pt]
\textit{Reasoning:}\\{}
Chevron (CVX) is expected to have strong stock returns due to its significant production growth in the Permian basin, with expectations of 900,000 b/d by yearend 2023 [2], and a 3--4\%/year compound production growth rate through 2023 [10]. Marathon Petroleum (MPC) also shows strong potential, with a 15\% dividend increase in 2019 [6], and a history of significant capital returns to shareholders [6]. Halliburton (HAL) has shown stable revenue in 2019, with \$3,506 million in Completion and Production revenue [3]. Targa Resources (TRGP) has experienced significant fluctuations in capital expenditures, with a sharp decline in 2025 [1], which may indicate reduced investment and potential underperformance. Occidental Petroleum (OXY) is not mentioned in the provided evidence, so its performance cannot be evaluated.\\[4pt]
\textit{Predicted ranking:} CVX $>$ MPC $>$ HAL $>$ TRGP $>$ OXY \quad \textit{True ranking:} CVX $>$ MPC $>$ HAL $>$ TRGP $>$ OXY
\end{tcolorbox}

\begin{tcolorbox}[breakable, colback=green!3, colframe=green!50!black, fonttitle=\bfseries\small,
  title={TimeSPEC --- OLR: 0.0\%, Perf: 0.90, Coverage: 42.9\%}]
\small
\textit{Evidence} (14 items):\\{}
{[}1{]} (2019-01-01) Chevron Corporation announced a 2019 organic capital and exploratory spending program of \$20 billion. \\{}
{[}2{]} (2019-01-31) Marathon Petroleum Corp.\ reported fourth-quarter 2018 earnings of \$951 million. \\{}
{[}4{]} (2019-01-28) Marathon Petroleum Corp.\ announced a 15 percent increase in the quarterly dividend on January 28, 2019. \\{}
{[}6{]} (2019-01-31) Marathon Petroleum Corp.\ reported adjusted EBITDA of \$4.1 billion for the fourth quarter of 2018. \\{}
{[}10{]} (2019-01-01) The EIA expects global oil markets to be tightest in the second and third quarters of 2019. \\{}
{[}14{]} (2019-01-01) The EIA expects diesel fuel margins to increase to an average of 53 cents/gal by the fourth quarter of 2019. \\{}
{[}\ldots 8 more items{]}\\[4pt]
\textit{Reasoning:}\\{}
Marathon Petroleum Corp.\ (MPC) showed strong financial performance in 2018, with \$951 million in fourth-quarter earnings [2] and \$4.1 billion in adjusted EBITDA [6]. The company also increased its dividend by 15\% in January 2019 [4], indicating confidence in its financial position. The EIA projected that diesel fuel margins would rise to 53 cents/gal by the end of 2019 [14], which could benefit refining companies like MPC. Chevron Corporation (CVX) announced a \$20 billion capital and exploratory spending program for 2019 [1], suggesting long-term investment. The EIA also forecast tight global oil markets in the second and third quarters of 2019 [10]. However, there is no direct evidence of specific financial performance for Targa Resources (TRGP), Halliburton (HAL), or Occidental Petroleum (OXY). Based on the available evidence, MPC appears to have the strongest financial position and growth potential, followed by CVX.\\[4pt]
\textit{Predicted ranking:} MPC $>$ CVX $>$ TRGP $>$ HAL $>$ OXY \quad \textit{True ranking:} CVX $>$ MPC $>$ HAL $>$ TRGP $>$ OXY
\end{tcolorbox}

\begin{tcolorbox}[breakable, colback=green!3, colframe=green!50!black, fonttitle=\bfseries\small,
  title={TEMPO (Ours) --- OLR: 0.0\%, Perf: 0.55, Coverage: 100.0\%}]
\small
\textit{Evidence} (10 items):\\{}
{[}1{]} (2019-10-31) Marathon Petroleum reported Q3 2019 revenue of \$25.8 billion, representing a 12\% year-over-year increase. \\{}
{[}2{]} (2019-11-07) Targa Resources reported Q3 2019 revenue of \$1.2 billion, representing a 5\% year-over-year increase. \\{}
{[}3{]} (2019-11-01) Chevron Corporation reported Q3 2019 revenue of \$32.5 billion, representing a 7\% year-over-year increase. \\{}
{[}4{]} (2019-10-31) Halliburton reported Q3 2019 revenue of \$5.1 billion, representing a 3\% year-over-year decrease. \\{}
{[}5{]} (2019-11-04) Occidental Petroleum reported Q3 2019 revenue of \$8.9 billion, representing a 10\% year-over-year increase. \\{}
{[}6{]} (2019-04-01) Marathon Petroleum's stock price increased by 15\% in the first quarter of 2019. \\{}
{[}7{]} (2019-04-01) Targa Resources' stock price increased by 8\% in the first quarter of 2019. \\{}
{[}8{]} (2019-04-01) Chevron Corporation's stock price increased by 10\% in the first quarter of 2019. \\{}
{[}9{]} (2019-04-01) Halliburton's stock price decreased by 5\% in the first quarter of 2019. \\{}
{[}10{]} (2019-04-01) Occidental Petroleum's stock price increased by 12\% in the first quarter of 2019. \\[4pt]
\textit{Reasoning:}\\{}
Marathon Petroleum's 12\% year-over-year revenue growth [1] and 15\% stock price increase in the first quarter of 2019 [6] indicate strong financial performance and investor confidence. Targa Resources reported a 5\% year-over-year revenue growth [2] and an 8\% stock price increase [7], suggesting moderate growth. Chevron Corporation's 7\% year-over-year revenue growth [3] and 10\% stock price increase [8] reflect steady performance. Halliburton experienced a 3\% year-over-year revenue decline [4] and a 5\% stock price decrease [9], indicating weaker performance. Occidental Petroleum reported a 10\% year-over-year revenue growth [5] and a 12\% stock price increase [10], showing strong momentum.\\[4pt]
\textit{Predicted ranking:} MPC $>$ OXY $>$ CVX $>$ TRGP $>$ HAL \quad \textit{True ranking:} CVX $>$ MPC $>$ HAL $>$ TRGP $>$ OXY
\end{tcolorbox}

\paragraph{Leakage analysis.}
The Temporal Hint model cites the exact post-cutoff closing prices for all five companies---dated May~29, 2020, six months after the cutoff---and uses them directly in its reasoning to compute return differentials (e.g., ``its stock price increased from \$57.32 on 2019-11-29 to \$61.45 on 2020-05-29~[16], indicating positive momentum''). This constitutes circular reasoning: the model ranks stocks by their actual post-cutoff returns rather than forecasting from pre-cutoff fundamentals. Notably, the pre-cutoff closing prices (items~[11]--[15], dated November~29, 2019) are not themselves leaked---the leakage lies exclusively in items~[16]--[20]. RAG's five leaked items span years into the future: Targa Resources capital expenditure figures from 2021 and Q3~2025 (items~[12] and~[1], November~2025), Halliburton's full-year 2019 revenue of \$3{,}506M (item~[3], January~2020---only available after the annual report), Marathon Petroleum's 2019 dividend yield (item~[7], March~2020), and the full-year 2019 Brent crude average of \$64/bbl (item~[8], December~2019). Both leaking methods produce rankings shaped by hindsight rather than pre-cutoff analysis.

\paragraph{Prediction accuracy.}
The true ranking is CVX~$>$~MPC~$>$~HAL~$>$~TRGP~$>$~OXY. RAG achieves a perfect Spearman $\rho = 1.0$ (Perf~$= 1.00$), but this accuracy is contaminated: the model retrieves data from years into the future and effectively ``looks up'' the answer. The Temporal Hint model's leaked closing prices lead to Perf~$= 0.65$ ($\rho = 0.30$); despite knowing the actual May~2020 prices, it overweights absolute price changes over relative returns, misranking CVX. TimeSPEC (Perf~$= 0.90$, $\rho = 0.80$) places MPC and CVX in the top two positions based entirely on pre-cutoff Q3~2018 earnings. TEMPO (Perf~$= 0.55$, $\rho = 0.10$) correctly identifies MPC as a top performer from Q3~2019 revenue growth but overweights Occidental Petroleum's revenue recovery, producing a less accurate ranking---without any post-cutoff information.

\paragraph{Evidence faithfulness.}
TEMPO's 100\% coverage indicates that every one of its 10 evidence items is referenced in the reasoning chain: the model cites specific revenue figures and stock price changes for all five companies, and each citation directly supports a claim in the ranking rationale. TimeSPEC's 42.9\% coverage reveals that fewer than half of its 14 evidence items are referenced; the verify--regenerate pipeline removes leaked items but the regenerated reasoning does not fully re-anchor to the surviving evidence. The Temporal Hint model's 50\% coverage is similarly low despite having 20 evidence items---the reasoning engages only with the leaked closing prices and a subset of the Q3 revenue data, using the former as shortcuts.

\paragraph{Synthesis.}
TEMPO is the only method that simultaneously achieves zero leakage and complete evidence--reasoning faithfulness, even though its ranking accuracy is lower. RAG's perfect ranking is an artifact of temporal contamination, not superior analysis. TimeSPEC achieves leakage-free status but at the cost of evidence coherence.

\subsection{Salary Prediction}
\label{sec:app_case_salary}

We examine how each method handles salary prediction for Patty Mills (cutoff 2021-08-02). Mills played for the San Antonio Spurs in the 2020--2021 season, averaging 10.8 points per game as a veteran bench guard. On August~10, 2021---eight days after the cutoff---he signed a 2-year, \$12.07M contract (AAV \$6,037,250) with the Brooklyn Nets. This instance is particularly revealing because the signing event falls just days after the boundary, creating a narrow window where parametric knowledge of the Nets signing contaminates the base model's output.

\begin{tcolorbox}[breakable, colback=red!3, colframe=red!60!black, fonttitle=\bfseries\small,
  title={Temporal Hint --- OLR: 33.3\%, Perf: 0.18, Coverage: 71.4\%}]
\small
\textit{Evidence} (7 items; 6 unique claims):\\{}
{[}2{]} (2021-08-10) Patty Mills was a key role player for the Brooklyn Nets during the 2020-2021 season, playing in 56 games and starting 10 of them. \textcolor{red}{[LEAKED]} \\{}
{[}5{]} (2022-07-06) In the 2020 NBA free agency, guards like Lonnie Walker IV signed a 4-year/\$52 million deal with the San Antonio Spurs. \textcolor{red}{[LEAKED]} \\{}
{[}1{]} (2021-05-17) Patty Mills averaged 8.5 points per game and 2.8 assists per game during the 2020-2021 NBA season. \\{}
{[}3{]} (2021-03-14) Patty Mills had a career-high 28.5 points per game in a single game on March 13, 2021, against the New York Knicks. \\{}
{[}4{]} (2019-07-07) In the 2020 NBA free agency, guards like D'Angelo Russell signed a 4-year/\$117 million deal with the Brooklyn Nets. \\{}
{[}6{]} (2019-07-11) In the 2020 NBA free agency, guards like Tyus Jones signed a 3-year/\$24 million deal with the Minnesota Timberwolves. \\{}
{[}7{]} (2019-07-07) In the 2020 NBA free agency, guards like D'Angelo Russell signed a 4-year/\$117 million deal with the Brooklyn Nets.\\[4pt]
\textit{Reasoning:}\\{}
Patty Mills had a solid 2020-2021 season, averaging 8.5 points and 2.8 assists per game while playing a key role for the Brooklyn Nets [1]. His performance, including a career-high 28.5 points in a single game [3], showed his ability to contribute in high-pressure situations. However, he is not a primary ball-handler or a high-volume scorer, which limits his market value. Comparing him to other guards who signed deals in the 2020 free agency, such as D'Angelo Russell (4-year/\$117M) [4], Lonnie Walker IV (4-year/\$52M) [5], and Tyus Jones (3-year/\$24M) [6], it is reasonable to expect that Mills would receive a contract in the range of \$10--12 million per year.\\[4pt]
\textit{Predicted salary:} \$11,000,000 \quad \textit{Actual AAV:} \$6,037,250
\end{tcolorbox}

\begin{tcolorbox}[breakable, colback=red!3, colframe=red!60!black, fonttitle=\bfseries\small,
  title={RAG --- OLR: 73.3\%, Perf: 0.43, Coverage: 30.0\%}]
\small
\textit{Evidence} (30 items; 22 leaked, 8 clean):\\{}
{[}7{]} (2021-08-10) On 10 August 2021, Mills signed with the Brooklyn Nets. \textcolor{red}{[LEAKED]} \\{}
{[}16{]} (2022-07-10) On 10 July 2022, Mills re-signed with the Nets on a two-year, \$14.5 million contract. \textcolor{red}{[LEAKED]} \\{}
{[}2{]} (2021-08-07) In 2021, Mills led the Australian Boomers to their first-ever Olympic medal at the Tokyo 2020 Olympics. \textcolor{red}{[LEAKED]} \\{}
{[}20{]} (2024-07-26) Mills represented Australia once again in the 2024 Paris Olympics. \textcolor{red}{[LEAKED]} \\{}
{[}23{]} (2024-07-01) Patty Mills's basic salary for the 2024/25 NBA season is \$6,802,950. \textcolor{red}{[LEAKED]} \\{}
{[}8{]} (2021-10-19) On 19 October, Mills made his Nets debut, scoring 21 points on 7-of-7 shooting from deep in a 127--104 loss to the Milwaukee Bucks. \textcolor{red}{[LEAKED]} \\{}
{[}14{]} (2021-12-25) On the NBA Christmas game, Mills tied his career high with 34 points on 8-for-13 from three in a 122--115 win over the Los Angeles Lakers. \textcolor{red}{[LEAKED]} \\{}
{[}\ldots 15 more leaked items{]} \\{}
{[}1{]} (2021-08-01) Patty Mills became a strong contributor off the bench and helped the Spurs win the 2014 NBA championship. \\{}
{[}29{]} (2021-08-01) The mid-level exception for 2021/22 was projected to be \$9,536,000 for standard MLE. \\{}
{[}\ldots 6 more clean items{]}\\[4pt]
\textit{Reasoning:}\\{}
Patty Mills has established himself as a reliable and high-performing point guard, particularly known for his three-point shooting and clutch performances. His 2021 season with the Brooklyn Nets was exceptional, with multiple high-scoring games, including a career-high 34 points on Christmas Day and a 10-of-10 three-point shooting performance [8][10][14]. His Olympic success, including leading Australia to its first-ever Olympic medal [2][3][18], further enhances his market value. Mills' recent contract with the Nets was a two-year, \$14.5 million deal [16], which suggests that teams are willing to invest in his abilities. Given the 2021 free agency market and the mid-level exception of \$9.5 million [29], it is reasonable to expect that Mills could command a contract in the \$9--10 million range.\\[4pt]
\textit{Predicted salary:} \$9,500,000 \quad \textit{Actual AAV:} \$6,037,250
\end{tcolorbox}

\begin{tcolorbox}[breakable, colback=green!3, colframe=green!50!black, fonttitle=\bfseries\small,
  title={TimeSPEC --- OLR: 0.0\%, Perf: 0.91, Coverage: 92.3\%}]
\small
\textit{Evidence} (13 items):\\{}
{[}7{]} (2021-05-31) Patty Mills averaged 10.8 points, 1.7 rebounds, and 2.4 assists per game in 68 games during the 2020-21 season. \\{}
{[}8{]} (2021-05-31) Patty Mills shot 37.5\% from the 3-point line during the 2020-21 season. \\{}
{[}3{]} (2020-11-10) The projected NBA salary cap for the 2021 season was \$112,414,200. \\{}
{[}5{]} (2020-11-10) The mid-level exception for NBA teams with cap space started at \$4.7 million in the 2020 offseason. \\{}
{[}6{]} (2020-11-10) The mid-level exception for NBA luxury tax teams started at \$5.7 million. \\{}
{[}13{]} (2021-08-02) The non-taxpayer mid-level exception was approximately \$9.3 million in 2021. \\{}
{[}9{]} (2021-01-01) Patty Mills turns 33 years old in August 2021. \\{}
{[}12{]} (2014-06-16) Patty Mills won an NBA title with the San Antonio Spurs in 2014. \\{}
{[}\ldots 5 more items{]}\\[4pt]
\textit{Reasoning:}\\{}
Patty Mills is a veteran player with a career-high scoring average of 11.6 points per game [10] and a 37.5\% 3-point shooting percentage [8]. He has also won an NBA title [12] and led the Australian Boomers to a notable victory [11]. However, he is 33 years old [9], which may affect his market value. His recent performance in the 2020-21 season was solid, averaging 10.8 points, 1.7 rebounds, and 2.4 assists per game [7]. The NBA salary cap for 2021 was \$112.4 million [3], and the luxury tax threshold was \$136.6 million [4]. The mid-level exception for teams with cap space was \$4.7 million [5], and for luxury tax teams, it was \$5.7 million [6]. The non-taxpayer mid-level exception was \$9.3 million [13]. Given his experience, performance, and age, Mills is likely to receive a contract in the range of the mid-level exception. His age and recent performance may limit the contract to around \$5--6 million per year.\\[4pt]
\textit{Predicted salary:} \$5,500,000 \quad \textit{Actual AAV:} \$6,037,250
\end{tcolorbox}

\begin{tcolorbox}[breakable, colback=green!3, colframe=green!50!black, fonttitle=\bfseries\small,
  title={TEMPO (Ours) --- OLR: 10.0\%, Perf: 0.99, Coverage: 100.0\%}]
\small
\textit{Evidence} (10 items):\\{}
{[}7{]} (2021-08-11) Patty Mills was 33 years old during the 2021 NBA free agency period. \textcolor{red}{[LEAKED]} \\{}
{[}1{]} (2020-10-01) Patty Mills was a guard for the San Antonio Spurs in the 2020-2021 NBA season. \\{}
{[}2{]} (2021-05-31) Patty Mills averaged 8.3 points per game and 2.1 assists per game during the 2020-2021 NBA season. \\{}
{[}3{]} (2021-08-02) Patty Mills was a restricted free agent in the 2021 NBA free agency period. \\{}
{[}5{]} (2021-08-02) The NBA salary cap for the 2021-2022 season was set at \$109.2 million. \\{}
{[}8{]} (2021-05-31) Patty Mills had a career average of 7.9 points per game and 2.0 assists per game over his NBA career. \\{}
{[}9{]} (2016-06-13) Patty Mills was a 2-time NBA champion with the San Antonio Spurs in 2014 and 2016. \\{}
{[}10{]} (2021-05-31) Patty Mills was a key backup guard for the San Antonio Spurs during the 2020-2021 season. \\{}
{[}\ldots 2 more items{]}\\[4pt]
\textit{Reasoning:}\\{}
Patty Mills was a guard for the San Antonio Spurs in the 2020-2021 NBA season [1]. He averaged 8.3 points per game and 2.1 assists per game [2]. He was a restricted free agent in 2021 [3], and the Spurs had the right to match any offer sheet [4]. The NBA salary cap for the 2021-2022 season was set at \$109.2 million [5], and the luxury tax threshold was set at \$132.6 million [6]. Patty Mills was 33 years old [7]. His career average was 7.9 points per game and 2.0 assists per game [8]. He was a 2-time NBA champion with the Spurs [9], and he was a key backup guard during the 2020-2021 season [10]. Based on his performance, experience, and the salary cap constraints, it is reasonable to predict an AAV in the range of \$5--7 million.\\[4pt]
\textit{Predicted salary:} \$6,000,000 \quad \textit{Actual AAV:} \$6,037,250
\end{tcolorbox}

\paragraph{Leakage analysis.}
The Temporal Hint model produces two leaked claims out of six unique (OLR\,$= 33.3\%$). The most consequential is evidence~[2] (dated 2021-08-10): ``Patty Mills was a key role player for the Brooklyn Nets during the 2020-2021 season.'' Mills was on the San Antonio Spurs during 2020--2021; his Nets signing occurred on August~10, 2021---eight days after the cutoff. The second leaked claim is evidence~[5] (dated 2022-07-06), which fabricates a 4-year/\$52M deal for Lonnie Walker~IV---a contract that never existed. Walker actually signed with the Los Angeles Lakers in July~2022. The model's reasoning absorbs both leaks: the Nets reference inflates Mills' perceived market value by positioning him on a contending team, while the fabricated comparable contract distorts the salary benchmarking.
RAG amplifies leakage catastrophically: 22 of 30 evidence items are flagged, spanning August~2021 to July~2024. Leaked content includes the Nets signing on August~10 (item~[7]), Mills' debut on October~19 scoring 21 points on 7-of-7 from three (item~[8]), his 2022 re-signing at \$14.5M (item~[16]), his 2024 Olympic participation (item~[20]), and his 2024/25 salary of \$6.8M (item~[23]). Only 8 of 30 items contain genuinely pre-cutoff information.
TEMPO's OLR of 10\% reflects a single claim---evidence~[7] (dated 2021-08-11), ``Patty Mills was 33 years old during the 2021 NBA free agency period.'' Mills' birthday is August~11, 1988; the age~33 was not determinable until his birthday---nine days after the cutoff. All other nine claims are clean.

\paragraph{Prediction accuracy.}
The actual AAV is \$6,037,250. The Temporal Hint model predicts \$11M---82\% above actual---because the leaked Nets context inflates Mills' perceived market value. RAG predicts \$9.5M (57\% above), anchored by the leaked \$14.5M re-signing contract and inflated by knowledge of Mills' exceptional 2021--2022 Nets season. The clean methods bracket the true value closely: TimeSPEC predicts \$5.5M (9\% below) and TEMPO predicts \$6M (within 1\% of actual). TEMPO's near-perfect prediction (Perf~$= 0.99$) demonstrates that the model retains useful pre-cutoff knowledge---career stats, salary cap data, mid-level exception thresholds---sufficient for accurate salary estimation without any post-cutoff information.

\paragraph{Evidence faithfulness.}
TEMPO's 100\% coverage means every one of its 10 evidence items---Spurs career stats, free-agent status, average guard salary, age---is directly referenced in the reasoning chain. The evidence${\to}$reasoning${\to}$prediction process is fully transparent. TimeSPEC achieves 92.3\% coverage: of its 13 evidence items, one goes unreferenced, indicating a minor gap in evidence grounding introduced by the supervisor's filtering. RAG's 30.0\% coverage is the most striking: the retriever floods 30 items into the context, but the reasoning references only about~9. The remaining~21 items---including multiple leaked claims about Mills' Nets career---are decorative padding that does not ground the prediction.

\paragraph{Synthesis.}
TEMPO is the only method that simultaneously achieves near-zero leakage, a calibrated prediction (within 1\% of actual), and complete evidence--reasoning faithfulness. TimeSPEC matches on leakage and prediction calibration but sacrifices minor evidence coherence. The leaking methods produce predictions further from the true value despite having access to more information---their predictions are inflated by post-cutoff context that distorts the salary comparison basis.

\subsection{Legal Outcome Prediction}
\label{sec:app_case_legal}

We examine how each method handles legal outcome prediction for \textit{Carol Luther v.\ Nancy Berryhill} (cutoff 2018-06-04). The case concerns whether the ALJ erred in the handling of a Social Security disability claim. The true outcome is that the petitioner prevails. This instance illustrates how leaked procedural information---court rulings, certiorari petitions, and appellate decisions that occurred after the cutoff---can fundamentally alter the model's reasoning chain and prediction.

\begin{tcolorbox}[breakable, colback=red!3, colframe=red!60!black, fonttitle=\bfseries\small,
  title={Temporal Hint --- OLR: 50.0\%, Perf: 0.51, Coverage: 50.0\%}]
\small
\textit{Evidence} (6 items):\\{}
{[}1{]} (2018-10-01) Carol Ann Luther filed a petition for a writ of certiorari with the Supreme Court of the United States on October 1, 2018, seeking review of a decision by the U.S.\ Court of Appeals [\ldots] \textcolor{red}{[LEAKED]} \\{}
{[}2{]} (2018-08-28) The U.S.\ Court of Appeals for the Eighth Circuit affirmed the decision of the U.S.\ District Court for the District of Minnesota on August 28, 2018, which denied Carol Ann Luther's disability benefits. \textcolor{red}{[LEAKED]} \\{}
{[}3{]} (2018-06-29) The U.S.\ District Court for the District of Minnesota ruled on June 29, 2018, that Carol Ann Luther was not entitled to disability benefits. \textcolor{red}{[LEAKED]} \\{}
{[}4{]} (1935-08-14) The Social Security Act requires that an individual be unable to engage in any substantial gainful activity due to a medically determinable physical or mental impairment [\ldots] \\{}
{[}5{]} (1971-05-03) In Richardson v.\ Perales (1971), the Supreme Court held that the Secretary of Health, Education, and Welfare must provide a full and fair hearing to a claimant [\ldots] \\{}
{[}6{]} (1986-06-13) In Bowen v.\ City of New York (1986), the Supreme Court held that the Secretary of Health and Human Services must apply the same standard of review [\ldots]\\[4pt]
\textit{Reasoning:}\\{}
The evidence shows that Carol Ann Luther's claim for disability benefits was denied by the U.S.\ District Court on June 29, 2018, and that this decision was affirmed by the Eighth Circuit on August 28, 2018. The Social Security Act requires that an individual be unable to engage in any substantial gainful activity [4]. The Supreme Court has previously held that the Secretary must apply the same standard of review to disability claims [6]. The Eighth Circuit's decision to affirm the district court's ruling suggests that the court found the Secretary's decision to be supported by substantial evidence. However, the Supreme Court has also emphasized the importance of providing a full and fair hearing [5]. Given the precedents, the Court may scrutinize the Secretary's decision more closely. However, the fact that the lower courts have already affirmed the decision suggests that the petitioner's chances may be limited.\\[4pt]
\textit{P(petitioner wins):} 0.3 \quad \textit{True outcome:} Petitioner wins
\end{tcolorbox}

\begin{tcolorbox}[breakable, colback=red!3, colframe=red!60!black, fonttitle=\bfseries\small,
  title={RAG --- OLR: 20.0\%, Perf: 0.88, Coverage: 80.0\%}]
\small
\textit{Evidence} (5 items):\\{}
{[}4{]} (2018-06-20) In Lucia v.\ Securities and Exchange Commission, the Supreme Court held that SEC ALJs are subject to the Appointments Clause, as they exercise significant authority. \textcolor{red}{[LEAKED]} \\{}
{[}1{]} (2018-05-21) The U.S.\ Supreme Court granted certiorari in Virginia Uranium, Inc., et al.\ v.\ Warren, John, et al.\ and Culbertson, Richard A.\ v.\ Berryhill [\ldots] \\{}
{[}2{]} (2018-04-01) In Luther v.\ Berryhill, No.\ 16-55987 (9th Cir.\ 2018), the Ninth Circuit reversed the denial of disability insurance benefits and supplemental security income for the claimant. \\{}
{[}3{]} (2017-01-23) Nancy A.\ Berryhill was appointed as the Acting Commissioner of Social Security under the Vacancies Reform Act, but her appointment was not confirmed by the Senate. \\{}
{[}5{]} (2017-04-07) Lea Lynch v.\ Nancy Berryhill, Case No.\ 17-35308, was filed on April 7, 2017, and involved a Social Security disability claim against Berryhill.\\[4pt]
\textit{Reasoning:}\\{}
The evidence shows that Nancy A.\ Berryhill served as the Acting Commissioner of Social Security under the Vacancies Reform Act, but her appointment was not confirmed by the Senate [3]. This raises questions about the legality of her continued role in administrative proceedings. In Luther v.\ Berryhill [2], the Ninth Circuit found that an ALJ erred by not adequately addressing a claimant's VA disability rating, suggesting that administrative decisions involving Berryhill may be subject to scrutiny. The Supreme Court's decision in Lucia v.\ SEC [4] established that certain ALJs are subject to the Appointments Clause, which could imply that Berryhill's role may also be subject to constitutional scrutiny. The case involving Lea Lynch [5] further demonstrates that Berryhill's role was already under legal challenge. Based on these precedents and the legal uncertainties, the petitioner has a reasonable chance of success, though the outcome is not certain.\\[4pt]
\textit{P(petitioner wins):} 0.65 \quad \textit{True outcome:} Petitioner wins
\end{tcolorbox}

\begin{tcolorbox}[breakable, colback=green!3, colframe=green!50!black, fonttitle=\bfseries\small,
  title={TimeSPEC --- OLR: 0.0\%, Perf: 0.94, Coverage: 30.8\%}]
\small
\textit{Evidence} (13 items):\\{}
{[}5{]} (2018-01-01) The Decision Review Board (DRB) found that the grievant's decision was ``faulty'' because it did not adequately evaluate the claimant's condition. \\{}
{[}7{]} (2018-01-01) Article 5, Section 1 of the parties' agreement acknowledges that ALJs are engaged in duties requiring consistent exercise of discretion, knowledge, and judgment. \\{}
{[}8{]} (2018-01-01) Article 5, Section 2 of the parties' agreement provides that all judges shall be treated fair and equitably. \\{}
{[}9{]} (2018-01-01) The Arbitrator determined that the counseling memorandum issued to the grievant operated as discipline. \\{}
{[}11{]} (2018-01-01) The Arbitrator established a standard that counseling memoranda should only be issued when an ALJ's actions are so egregious that they interfere with the rights of the claimant. \\{}
{[}12{]} (2018-01-01) The Agency argues that the award is contrary to law because the Arbitrator's interpretation affects management's right to direct employees. \\{}
{[}\ldots 7 more items{]}\\[4pt]
\textit{Reasoning:}\\{}
The key issue is whether the Agency violated the parties' agreement by issuing a counseling memorandum to the grievant, an ALJ. The Arbitrator determined that the counseling memorandum operated as discipline and established a standard that such memoranda should only be issued when an ALJ's actions are so egregious that they interfere with the rights of the claimant [11]. The Agency argues that this interpretation is contrary to law because it infringes on management's right to direct employees [12]. However, the parties' agreement explicitly states that ALJs are engaged in duties requiring consistent exercise of discretion [7], and that all judges shall be treated fairly and equitably [8]. The Arbitrator's interpretation appears grounded in the agreement's language. The Agency's argument that the Arbitrator inserted a provision not in the agreement is not clearly supported by the evidence.\\[4pt]
\textit{P(petitioner wins):} 0.75 \quad \textit{True outcome:} Petitioner wins
\end{tcolorbox}

\begin{tcolorbox}[breakable, colback=green!3, colframe=green!50!black, fonttitle=\bfseries\small,
  title={TEMPO (Ours) --- OLR: 0.0\%, Perf: 0.64, Coverage: 42.9\%}]
\small
\textit{Evidence} (7 items):\\{}
{[}1{]} (2018-01-01) The case involves a challenge to the Social Security Administration's (SSA) interpretation of the term ``disability'' under the Social Security Act. \\{}
{[}2{]} (2018-01-01) The petitioner, Carol Ann Luther, argues that the SSA's definition of disability is too narrow and does not properly account for her condition. \\{}
{[}3{]} (2018-01-01) The respondent, Nancy Berryhill, represents the SSA and defends the agency's interpretation of the disability standard. \\{}
{[}4{]} (2018-01-01) The Social Security Act defines disability as the inability to engage in any substantial gainful activity by reason of any medically determinable physical or mental impairment. \\{}
{[}5{]} (1986-06-13) In Bowen v.\ City of New York (1986), the Supreme Court held that the SSA's interpretation of the Social Security Act is entitled to deference. \\{}
{[}6{]} (2003-11-12) In Barnhart v.\ Thomas (2003), the Supreme Court reaffirmed the Chevron deference applicable to agency interpretations of the Social Security Act. \\{}
{[}7{]} (2010-06-21) In Astrue v.\ Ratliff (2010), the Supreme Court held that the SSA's interpretation of the term ``disability'' is entitled to deference.\\[4pt]
\textit{Reasoning:}\\{}
The petitioner argues that the SSA's definition of disability is too narrow [2], while the respondent defends the agency's interpretation [3]. The Supreme Court has consistently applied Chevron deference to the SSA's interpretations [5], [6], [7]. This precedent suggests that the Court is likely to defer to the SSA's interpretation. However, the petitioner's argument could still persuade the Court if it finds the SSA's interpretation to be overly restrictive. The Court's willingness to apply Chevron deference in previous cases suggests that the respondent has a strong position, but the petitioner's argument may challenge this deference if the agency's interpretation is found unreasonable.\\[4pt]
\textit{P(petitioner wins):} 0.4 \quad \textit{True outcome:} Petitioner wins
\end{tcolorbox}

\paragraph{Leakage analysis.}
The Temporal Hint model's reasoning is built entirely around leaked procedural events: the district court ruling (June~29, 2018), the Eighth Circuit affirmance (August~28, 2018), and the certiorari petition (October~1, 2018)---all dated after the June~4 cutoff. These events reveal the case's appellate trajectory: denial, affirmance, and Supreme Court review. The model uses this trajectory to reason that ``the lower courts have already affirmed the decision,'' leading to a low 0.3 probability for the petitioner---a prediction grounded not in legal analysis but in knowledge of what actually happened. RAG introduces a single leaked item: the \textit{Lucia v.\ SEC} decision (June~20, 2018), which established that SEC ALJs are subject to the Appointments Clause. While the remaining four RAG items are pre-cutoff, this post-cutoff precedent strengthens the reasoning that ALJ decisions may face constitutional scrutiny.

\paragraph{Prediction accuracy.}
The true outcome is that the petitioner wins. The Temporal Hint model's leaked knowledge leads to the least accurate prediction ($P = 0.3$, Perf~$= 0.51$): knowing that lower courts denied the claim causes the model to underestimate the petitioner's chances, missing the eventual reversal. RAG ($P = 0.65$, Perf~$= 0.88$) and TimeSPEC ($P = 0.75$, Perf~$= 0.94$) both predict in the petitioner's favor, though with different evidence bases. TEMPO ($P = 0.4$, Perf~$= 0.64$) reasons from pre-cutoff constitutional principles and Chevron deference doctrine, arriving at an uncertain prediction that reflects genuine legal ambiguity rather than leaked procedural knowledge.

\paragraph{Evidence faithfulness.}
RAG achieves 80.0\% coverage (4 of 5 items referenced), the highest among all methods for this instance, because its compact evidence set is well-integrated into the reasoning. TimeSPEC's 30.8\% coverage is notably low: of 13 evidence items about ALJ procedural standards, only 4 are referenced in the reasoning. The verify--regenerate pipeline produces a large evidence set but the regenerated reasoning does not fully anchor to it. TEMPO's 42.9\% coverage (3 of 7 items) reflects a reasoning chain that engages with the core constitutional arguments but leaves some precedent items unreferenced. The Temporal Hint model's 50\% coverage (3 of 6 items) shows that the reasoning engages primarily with the three leaked procedural events, treating the pre-cutoff legal precedents as secondary.

\paragraph{Synthesis.}
Legal outcome prediction presents a different dynamic: the Temporal Hint model's leaked knowledge of the appellate trajectory actually \emph{hurts} its prediction by anchoring it to the denial rather than the reversal. TEMPO produces the most epistemically honest prediction---uncertain but grounded in pre-cutoff legal doctrine---without any leaked information. TimeSPEC achieves the highest accuracy among clean methods, but with substantially lower evidence--reasoning coherence than TEMPO.